\documentclass[11pt,a4paper]{article}
\usepackage[utf8]{inputenc}
\usepackage[parfill]{parskip}    % Activate to begin paragraphs with an empty line rather than an indent
\usepackage{graphicx}
\usepackage{amssymb, amsmath, amsthm}
\usepackage{epstopdf}
\usepackage{multibib}
\usepackage{mathrsfs}
\usepackage[marginratio=1:1,height=600pt,width=460pt,tmargin=100pt]{geometry}
\usepackage{layout, color}
\usepackage{bbm}

\usepackage{authblk}

\usepackage{algorithm}
\usepackage{algpseudocode}

\usepackage{setspace}
\usepackage{hyperref}

\usepackage{url}
\usepackage[caption=false]{subfig}
\usepackage{mdwlist}
\usepackage{multirow}
\usepackage{amsthm}
\usepackage{color}

\newcommand{\R}{\mathbb{R}}

\newtheorem{theorem}{Theorem}[section]
\newtheorem{proposition}[theorem]{Proposition}

\newtheorem{lemma}[theorem]{Lemma}

\newtheorem{remark}{Remark}
\newtheorem{conjecture*}{Conjecture}

\theoremstyle{plain}

\title{
RotDCF: Decomposition of Convolutional Filters for Rotation-Equivariant Deep Networks
}

\author[1]{Xiuyuan Cheng\footnote{Email: xiuyuan.cheng@duke.edu.}}
\author[2]{Qiang Qiu}
\author[1,2]{Robert Calderbank}
\author[2]{Guillermo Sapiro}

\affil[1]{Department of Mathematics, Duke University}
\affil[2]{Department of Electrical and Computer Engineering, Duke University}

\date{}

\begin{document}

\maketitle

\begin{abstract}

Explicit encoding of group actions in deep features makes it possible for convolutional neural networks (CNNs) to handle global deformations of images, 
which is critical to success in many vision tasks. 
This paper proposes to decompose the convolutional filters over joint steerable bases across the space and the group geometry simultaneously,
namely a rotation-equivariant CNN with decomposed convolutional filters (RotDCF).
This decomposition facilitates computing the joint convolution,
which is proved to be necessary for the group equivariance.
It significantly reduces the model size and computational complexity while preserving performance,
and truncation of the bases expansion serves implicitly to regularize the filters. 
On datasets involving in-plane and out-of-plane object rotations,
RotDCF deep features demonstrate greater robustness and interpretability than regular CNNs.
The stability of the equivariant representation to input variations is also proved theoretically 
under generic assumptions on the filters in the decomposed form.
The RotDCF framework can be extended to groups other than rotations, 
providing a general approach which achieves both group equivariance and representation stability at a reduced model size.
\end{abstract}

%%%%%%%%%%%%%
%%
\section{Introduction}

While deep convolutional neural networks (CNN) have been widely used in computer vision and image processing applications, 
they are not designed to handle large group actions like rotations,
which degrade the performance of CNN in many tasks 
\cite{cheng2016rifd,hallman2015oriented,jaderberg2015spatial,laptev2016ti,maninis2016convolutional}. 
The regular convolutional layer is equivariant to input translations, 
but not other group actions. 
An indirect way to encode group information into the deep representation
 is to conduct generalized convolutions across the group as well, as in \cite{Cohen2016}.
In theory, this approach can guarantee the 
group equivariance of the learned representations,
 which provides better interpretability and regularity 
as well as the capability of estimating the group action in localization, boundary detection, etc.
 
 For the important case of 2D rotations,
group-equivariant CNNs have been constructed in several recent works,
e.g., \cite{weiler2017learning}, Harmonic Net \cite{worrall2017harmonic} and
Oriented Response Net \cite{zhou2017oriented}.
In such networks, 
the layer-wise output has an extra index representing the group element (c.f. Section \ref{sec:2-1}, Table \ref{tab:snn}),
and consequently,
the convolution must be across the space and the group jointly (proved in Section \ref{sec:3-1}). 
This typically incurs a significant increase in the number of parameters and computational load,
even with the adoption of steerable filters \cite{freeman1991design,weiler2017learning,worrall2017harmonic}. 
In parallel, low-rank factorized filters have been proposed 
for sparse coding as well as the compression and regularization of deep networks.
In particular, 
\cite{qiu2018dcfnet} showed that 
decomposing filters under non-adaptive bases can be an effective way to reduce the model size of CNNs without sacrificing performance.
However, these approaches do not directly apply to be group-equivariant.
We review these connections in more detail in Section \ref{sec:1-1}.

This paper proposes a truncated bases decomposition of the filters in group-equivariant CNNs,
which we call the rotation-equivariant CNN with decomposed convolutional filters (RotDCF).
Since we need a joint convolution over $\R^2$ and $SO(2)$,
the bases are also joint across the two geometrical domains, c.f. Figure \ref{fig:rotdcf1}.
The benefits of bases decomposition are three-fold:
\begin{itemize}
\vspace{- .05 in}
\item[(1)] 
Reduction of the number of parameters and computational complexity of rotation-equivariant CNNs, c.f. Section \ref{sec:2-3};
\vspace{- .1 in}
\item[(2)] 
Implicit regularization of the convolutional filters,
leading to improved robustness of the learned deep representation  shown experimentally in Section \ref{sec:4};
\vspace{- .1 in}
\item[(3)]  
Theoretical guarantees on stability of the equivariant representation to input deformations,
that follow from a more generic condition on the filters in the decomposed form, 
c.f. Section \ref{sec:3-2} and the Appendix.
\vspace{- .05 in}
\end{itemize}
We explain this in more detail in the rest of the paper.

\begin{figure}[t]
%\vskip -0.2in
  \centering
\includegraphics[width=.99\linewidth]{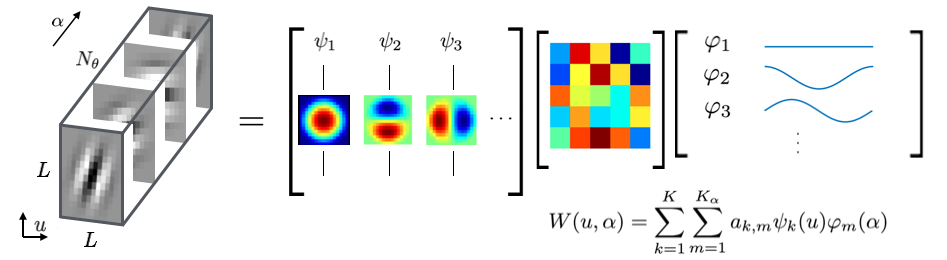} 
\caption{
\small 
Decomposition of the convolutional filter
across the 2D space (variable $u$) and the $SO(2)$ rotation group geometry (variable $\alpha$) simultaneously.
The filter is represented as a truncated expansion under the prefixed bases $\psi_k(u)\varphi_m(\alpha)$
with adaptive coefficients $a_{k,m}$ learned from data, where $\psi_k$ are Fourier-Bessel bases and $\varphi_m$ are Fourier bases.
The filter has $N_\theta$ group-indexed channels (indexed by $\alpha$)
and only one input and output unstructured channel (indexed by $\lambda'$ and $\lambda$ respectively in text) 
for simplicity, c.f. Section \ref{sec:2-1} and Table \ref{tab:snn}.
}
\vskip -0.1 in
\label{fig:rotdcf1}
\end{figure}

\subsection{Related Work}\label{sec:1-1}

{\bf Learning with factorized filters.}
In the sparse coding literature,  
\cite{rubinstein2010double} proposed the factorization of learned dictionaries 
under another prefixed dictionary. 
Separable filters were used in \cite{rigamonti2013learning} to learn the coding of images. 
 \cite{papyan2017convolutional} interpreted CNN as an iterated convolutional sparse coding machine,
and in this view, the factorized filters should correspond to a ``dictionary of the dictionary'' 
as in \cite{rubinstein2010double}. 
In the deep learning literature, 
low-rank factorization of convolutional filters has been previously used to remove redundancy in trained CNNs \cite{DentonZBLF14, jaderberg2014speeding}.
The compression of deep networks has also been studied in 
\cite{chen2015compressing, han2015deep_compression, han2015learning}, 
SqueezeNet \cite{SqueezeNet}, etc., where the low-rank factorization of filters can be utilized.
MobileNets \cite{howard2017mobilenets} used 
depth-wise separable convolutions to obtain significant compression.
Tensor decomposition of convolutional layers was used in \cite{lebedev2014speeding} for CPU speedup.
\cite{tai2015convolutional}  proposed low-rank-regularized filters and
obtained improved classification accuracy with reduced computation.
\cite{qiu2018dcfnet} studied decomposed-filter CNN with
prefixed bases and trainable expansion coefficients, 
showing that the truncated Fourier-Bessel bases decomposition incurs 
almost no decrease in classification accuracy 
while significantly reducing the model size and improving the robustness of the deep features.
None of the above networks are group equivariant. 

{\bf Group-equivariant deep networks.}
The encoding of group information into network representations has been studied extensively.
Among earlier works, transforming auto-encoders \cite{hinton2011transforming} 
used a non-convolutional network to learn group-invariant features
and compared with hand-crafted ones.
Rotation-invariant descriptors were studied in \cite{schmidt2012learning} with product models,
and in \cite{Jaderberg2015,Kivinen2011, Schmidt2012} by estimating the specific image transformation.  
\cite{Marcos2016,Wu2015} proposed rotating conventional filters to perform rotation-invariant texture and image classification.
The joint convolution across space and rotation has been studied in the scattering transform \cite{oyallon2015deep,sifre2013rotation}.
Group-equivariant CNN was considered by \cite{Cohen2016},
which handled several finite small-order discrete groups on the input image.
Rotation-equivariant CNN was later developed in \cite{weiler2017learning,worrall2017harmonic,zhou2017oriented} and elsewhere. 
In particular, steerable filters were used in \cite{cohen2016steerable, weiler2017learning,worrall2017harmonic}. 
$SO(3)$-equivariant CNN for signals on spheres was studied in \cite{cohen2018spherical} in a different setting.
Overall, the efficiency of equivariant CNNs remains to be improved since the model is typically several times larger than that of a regular CNN.

%%%%%%%%%%%%%
%%
\section{Rotation-equivariant DCF Net}\label{sec:2}

\begin{table}[b]
%\vskip -0.1  in
\begin{centering}
\scriptsize
\begin{tabular}[t]{ c | c  | c}
\hline
    fully-connected layer 	                               &  regular convolutional layer                                  &  CNN with group-indexed channels \\
 \hline
                                                                           &                                                                           &

    \\
    $ x^{(l-1)} (\lambda' ) \to x^{(l)}(\lambda ) $    &  $x^{(l-1)}(u',\lambda' ) \to x^{(l)}(u,\lambda )$ & $x^{(l-1)} (u',\alpha',\lambda' ) \to x^{(l)}(u,\alpha,\lambda )$
    \\
                                                                           &                                                                           &
    \\
\hline
    $\lambda' \to \lambda$: dense & $u' \to u$: spatial convolution  & $u' \to u$, $\alpha' \to \alpha$: joint convolution \\  
                                                     & $\lambda' \to \lambda$: dense & $\lambda' \to \lambda$: dense \\                                                   
                                                     
 \hline
\end{tabular}
\caption{\label{tab:snn}
\small
Comparison of a fully-connected layer, 
a regular convolutional layer,
and a rotation-equivariant  convolutional layer with group-indexed channels. 
}
\end{centering}
\vskip -0.2 in
\end{table}

\subsection{Rotation-equivariant CNN}\label{sec:2-1}

A rotation-equivariant CNN indexes the channels by the $SO(2)$ group \cite{weiler2017learning,zhou2017oriented}:
The $l$-th layer output is written as 
$x^{(l)}(u,\alpha,\lambda)$, the position $u \in \R^2$, the rotation $\alpha \in S^1$, and $\lambda \in [M_l]$,
$M_l$ being the number of unstructured channel indices.
Throughout the paper $[m]$ stands for the set $\{1,\cdots, m\}$.
We denote the group $SO(2)$ also by the circle $S^1$ since the former
 is parametrized by
  the rotation angle.
The convolutional filter at the $l$-th layer is represented as 
$
W^{(l)}_{\lambda', \lambda}(v, \alpha)$, 
$\lambda' \in [M_{l-1}]$, 
$\lambda \in [M_l]$, 
$v \in \R^2$, 
$\alpha \in S^1$,
except for the 1st layer where there is no indexing of $\alpha$.
In practice, $S^1$ is discretized into $N_\theta$ points on $(0, 2\pi)$.
Throughout the paper we denote the summation over $u$ and $\alpha$ by continuous integration,
and the notation $\int_{S^1} (\cdots) d\alpha$ means $\frac{1}{2\pi}\int_0^{2\pi} (\cdots) d\alpha$. 

To guarantee group-equivariant representation (c.f. Section \ref{sec:3-1}), 
the convolution is jointly computed over $\R^2$ and $SO(2)$. 
Specifically, 
let the rotation by angle $t$ be $\Theta_t: \R^2 \to \R^2$,
in the 1st layer,
\begin{equation}\label{eq:joint-conv-1}
x^{(1)}(u,\alpha,\lambda)
=
\sigma\left( 
\sum_{\lambda'=1}^{M_{0}} \int_{S^1} \int_{\R^2}
x^{(0)}( u + v', \lambda') W^{(1)}_{\lambda', \lambda} ( \Theta_{\alpha} v' ) dv' + b^{(1)}(\lambda)
\right).
\end{equation}
For $l>1$, the joint convolution of $u$ and $\alpha$ takes the form
\begin{equation}\label{eq:joint-conv-2}
x^{(l)}(u,\alpha,\lambda)
=
\sigma\left( 
\sum_{\lambda'=1}^{M_{l-1}} \int_{S^1} \int_{\R^2}
x^{(l-1)}( u + v', \alpha', \lambda') W^{(l)}_{\lambda', \lambda}(\Theta_{\alpha} v', \alpha' - \alpha) dv' d\alpha' + b^{(l)}(\lambda)
\right).
\end{equation}
Table \ref{tab:snn} compares
a rotation-equivariant CNN and a regular CNN.

While group equivariance is a desirable property,
the model size and computation can be increased significantly 
due to the extra index $\alpha \in [N_{\theta}]$.
We will address this issue by introducing the bases decomposition of the filters. 

\subsection{Decomposed Filters Under Steerable Bases}

We decompose the filters with respect to $u$ and $\alpha$ simultaneously: 
Let $\{\psi_k\}_k$ be a set of bases on the unit 2D disk,
and $\{ \varphi_{m} \}_m$ be bases on $S^1$.
At the $l$-th layer,
 let $j_l$ be the scale of the filter in $u$,
 and $\psi_{j,k} = 2^{-2j} \psi_k( 2^{-j} u)$ 
 (the filter is supported on the disk of radius $2^{j_l}$).
 Since we use continuous convolutions, 
 the down-sampling by ``pooling" is modeled by the rescaling of the filters in space.
The decomposed filters are of the form
\begin{equation}\label{eq:al-2}
W^{(1)}_{\lambda', \lambda}(v)  
 = \sum_k a^{(1)}_{\lambda', \lambda}(k) \psi_{j_1, k}(v), 
~~
W^{(l)}_{\lambda', \lambda}(v, \beta)  
 = \sum_k \sum_m a^{(l)}_{\lambda', \lambda}(k,m) \psi_{j_l, k}(v) \varphi_m( \beta),~ l>1,
\end{equation}
which is illustrated in Figure \ref{fig:rotdcf1} (for $l > 1$).
We use Fourier-Bessel (FB) bases for $\{ \psi_{k} \}_k$ which are steerable, 
and Fourier bases for  $\{ \varphi_{m}\}_m$,
so that the operation of rotation is a diagonalized linear transform under both bases. 
Specifically, in the complex-valued version,
\begin{equation}\label{eq:rot-alg-a}
\psi_k( \Theta_ t  v) = e^{- i m(k) t} \psi_k(v), ~ \forall k, 
~~~
\varphi_m(  \alpha -t ) = e^{- i m t} \varphi_m(  \alpha  ), ~ \forall m.
\end{equation}
This means that after the convolutions on $\R^2 \times S^1$ with the bases $\psi_k( v) \varphi_l(\alpha)$
are computed for all $k$ and $l$, both up to certain truncation, 
the joint convolution \eqref{eq:joint-conv-1}, \eqref{eq:joint-conv-2} with all rotated filters 
can be calculated by the algebraic manipulation of the expansion coefficients $a^{(l)}_{\lambda', \lambda}(k,m)$,
and without any re-computation of the spatial-rotation joint convolution.
Standard real-valued versions of the bases $\psi_k$ and $\varphi_m$ in $\sin$'s and $\cos$'s are used in practice.
During training, only the expansion coefficients $a$'s are updated, and the bases are fixed.

Apart from the saving of parameters and computation, which will be detailed next, 
the bases truncation also regularizes
the convolutional filters by discarding the high frequency components.
As a result, RotDCF Net reduces response to those components in the input at all layers, 
which barely affects recognition performance and improves the robustness of the learned feature.
The theoretical properties of RotDCF Net, particularly the representation stability, will be analyzed in Section \ref{sec:3}. 

\subsection{Numbers of Parameters and Computation Flops}\label{sec:2-3}

In this section we analyze a single convolutional layer,
and numbers for specific networks are shown in Section \ref{sec:4}.
Implementation and memory issues will be discussed in Section \ref{sec:5}.

{\bf Number of trainable parameters}: 
In a regular CNN, a convolutional layer of size $L \times L \times M_0' \times M_0$ has $L^2M_0'M_0$ parameters.
In an equivariant CNN, a joint convolutional filter is of size $L \times L \times N_\theta \times M' \times M$, 
so that the number of parameters is $L^2 N_\theta M' M$.
In a RotDCF Net, $K$ bases are used in space and $K_\alpha$ bases across the angle $\alpha$,
so that the number of parameters is $K K_\alpha M' M$.
This gives a reduction of $\frac{K}{L^2}\cdot \frac{K_\alpha}{ N_\theta}$
compared to non-bases equivariant CNN.
In practice, after switching from a regular CNN to a RotDCF Net, 
typically  $M \le \frac{1}{2}M_0$ or more due to the adoption of filters in all orientations. 
The factor $\frac{K}{L^2}$ is usually between $\frac{1}{8}$ and $\frac{1}{3}$ depending on the network and the problem \cite{qiu2018dcfnet}.
In all the experiments in Section \ref{sec:4}, $K_\alpha$ is typically 5, and $N_\theta = 8$ or 16.
This means that RotDCF Net achieves a significant parameter reduction from the non-bases equivariant CNN,
and even reduces parameters from a regular CNN by a factor of $\frac{1}{2}$ or more.
 
{\bf Computation in a forward pass}:
When the input and output are both $W\times W$ in space, the forward pass  in a regular convolutional layer needs $M_0'M_0W^2 (1+2L^2)$ $\sim 2L^2 M_0'M_0W^2$
flops. (Each convolution with a $L \times L $ filter takes $2L^2W^2$, and there are $M_0'M_0$ convolution operations,
plus that the summation over $\lambda'$ takes $W^2M_0'M_0$ flops.)
In a rotation equivariant CNN without using bases, an convolutional layer would take $\sim 2 M'MW^2 L^2 N_\theta^2$ flops. 
In a RotDCF layer, the computation consists of three parts:
(1) The inner-product with $\varphi_m$ bases takes $W^2M' \cdot 2N_\theta K_\alpha$ flops.
(2) The spatial convolution with the $\psi_k$ bases takes $ K_\alpha M' K \cdot 2L^2W^2$ flops.
(3) The multiplication with $a_{\lambda',\lambda}(k,m) e^{-im(k) \alpha-i m \alpha} $
and summation over $\lambda', k, m$ takes
$M N_\theta (4 K K_\alpha M' + 2 W^2 K K_\alpha M')$ flops (real-valued version).
Putting together, the total is 
$2M'W^2K_\alpha (N_\theta + L^2 K + M N_\theta K)$,
and when $M$ is large, the third term dominates and it gives $2M'MW^2 K_\alpha K N_\theta $.
Thus the reduction by using bases-decomposed filters is again a factor of $\frac{K}{L^2}\cdot \frac{K_\alpha}{ N_\theta}$,
and the relative ratio with a regular CNN is about
$\frac{M' M }{M_0' M_0}\cdot \frac{K K_\alpha N_\theta}{L^2}$.

In summary, RotDCF Net achieves a reduction of $\frac{K}{L^2}\cdot \frac{K_\alpha}{ N_\theta}$ from non-bases equivariant CNNs,
 in terms of both model size and computation. 
With typical network architectures, 
RotDCF Net may be of a smaller model size than regular CNNs. 

%%%%%%%%%%%%%
%%
\section{Theoretical Analysis of Deep Features}\label{sec:3}

This section presents two analytical results: 
(1) Joint convolution \eqref{eq:joint-conv-1}, \eqref{eq:joint-conv-2} is sufficient and actually necessary to obtain rotation equivariance;
(2) Stability of the equivariant representation with respect to input variations is proved under generic conditions.
This is important in practice since rotations are never perfect.

\subsection{Group-equivariant Property}\label{sec:3-1}

Suppose that the input image undergoes some arbitrary rotation, 
and consider the effect on the $l$-th layer output.
Let rotation around point $u_{0}$ by angle $t$ be denoted by
$\rho=\rho_{u_{0},t}$, i.e.
$\rho_{u_0, t} u=u_{0}+\Theta_{t}(u-u_{0})$, for any $ u \in \R^2$,
and the transformed image by
$D_\rho x^{(0)}(u,\lambda) = x^{(0)}(\rho_{u_0,t}u, \lambda)$, for any $\lambda \in [M_0]$.
 We also define the action $T_\rho$ on the $l$-th layer output $x^{(l)}$, $l > 0$, as
\begin{equation}
\label{eq:def-Trho-l}
T_\rho x^{(l)}(u, \alpha, \lambda) = x^{(l)}( \rho_{u_0,t}u, \alpha - t, \lambda), 
\quad \forall \lambda \in [M_l].
\end{equation}

The following theorem, proved in Appendix, 
 shows that RotDCF Net produces group-equivariant features at all layers in the sense of 
\begin{equation}
\label{eq:so2-equiv}
x^{(l)}[ D_\rho x^{(0)}] =  T_\rho x^{(l)}[ x^{(0)}].
\end{equation}
Furthermore, 
the scheme defined in \eqref{eq:joint-conv-1}, \eqref{eq:joint-conv-2}
 is the unique design for a CNN with $SO(2)$-indexed channels that achieves \eqref{eq:so2-equiv}. 
In other words, the joint convolution \eqref{eq:joint-conv-1}, \eqref{eq:joint-conv-2} is necessary in our context. 

 \begin{theorem}\label{thm:equiv}
 In a RotDCF Net with $SO(2)$-indexed channels, 
 let $x^{(l)}[x^{(0)}]$ be the output at the $l$-th layer from input $x^{(0)}(u,\lambda)$.
 The relation \eqref{eq:so2-equiv} holds for all $l$ if and only if
 the convolutional layers are given by \eqref{eq:joint-conv-1}, \eqref{eq:joint-conv-2}.
 \end{theorem}

\subsection{Representation Stability under Input Variations}\label{sec:3-2}

{\bf Assumptions on the RotDCF layers.}
Following \cite{qiu2018dcfnet}, we make the following generic assumptions on the convolutional layers: 
First, 
\begin{itemize}
\item[ ]
{\bf (A1)} Non-expansive sigmoid: 
$\sigma: \R \to \R$ is non-expansive.
\end{itemize}
We also need a boundedness assumption on the convolutional filters $W^{(l)}$ for all $l$.
Specifically, define
\begin{equation}\label{eq:def-Al}
A_l : = 
\pi \max \{ 
	 \sup_{\lambda} \sum_{\lambda'=1}^{M_{l-1}} \| a^{(l)}_{\lambda', \lambda}\|_{\text{FB}}, \, 
	\sup_{\lambda'}  \frac{M_{l-1}}{M_l}  \sum_{\lambda=1}^{M_{l}} \| a^{(l)}_{\lambda', \lambda}\|_{\text{FB}}  \},
\end{equation}
where the weighted norm $\| \cdot \|_{\text{FB}}$ of $a_{\lambda', \lambda}^{(l)}$ is defined as 
\begin{equation}\label{eq:norm-al}
\| a^{(1)}_{\lambda', \lambda} \|_{\text{FB}}^2 
= \sum_k \mu_k  (a^{(1)}_{\lambda', \lambda}(k))^2, 
\quad
\| a^{(l)}_{\lambda', \lambda} \|_{\text{FB}}^2 
= \sum_k \sum_m \mu_k (a^{(l)}_{\lambda', \lambda}(k,m))^2, \quad l>1,
\end{equation}
$\mu_k$ being the Dirichlet Laplacian eigenvalues of the unit disk in $\R^2$. 
Second, we assume that
\begin{itemize}
\item[ ]
{\bf (A2)} Boundedness of filters: In all layers, $A_l \le 1$.
\end{itemize}
This implies a sequence of boundedness conditions on the convolutional filters in all layers, c.f. Proposition \ref{prop:Bl-Cl-Dl}.  
The validity of (A2) can be qualitatively fulfilled by normalization layers which is standard in practice. 
As the stability results in this section will be derived under (A2), 
this assumption motivates truncation of the bases expansion to only include low-frequency $k$ and $m$'s,
which is implemented in Section \ref{sec:4}. 

{\bf Non-expansiveness of the network mapping.}
Let the $L^2$ norm of $x^{(l)}$ be defined as 
\[
\|x^{(l)}\|^2 = \frac{1}{M_l} \sum_{\lambda=1}^{M_l} \frac{1}{|\Omega|}\int_{\R^2} \int_{S^1} x^{(l)}(u,\alpha,\lambda)^2 du  d\alpha,
\quad l\ge 1
\]
and $ \|x^{(0)}\|^2 = \frac{1}{M_0} \sum_{\lambda} \frac{1}{|\Omega|}\int_{\R^2} x^{(0)}(u, \lambda)^2 du $.
$\Omega$ is the domain on which $x^{(0)}$ is supported, usually $\Omega= [-1,1]\times [-1,1] \subset \R^2$.  
The following result is proved in Appendix:

\begin{proposition}\label{prop:l2stable}
In a RotDCF Net, under (A1), (A2), for all $l$,

(a) 
The mapping of the $l$-th convolutional layer (including $\sigma$), denoted as $x^{(l)}[x^{(l-1)}]$, is non-expansive,
i.e., $ \| x^{(l)}[ x_1 ] -  x^{(l)}[ x_2 ] \| \le \| x_1 - x_2\| $ for arbitrary $x_1$ and $x_2$. 

(b) $\| x_c^{(l)} \| \le \| x_c^{(l-1)} \|$ for all $l$, where 
$x_c^{(l)}(u,\alpha, \lambda) = x^{(l)}(u,\alpha, \lambda) - x_0^{(l)}(\lambda)$ (without index $\alpha$ when $l$=1) is the centered version of $x^{(l)}$
by removing $x_0^{(l)}$, defined to be the output at the $l$-th layer from a zero bottom-layer input.
 As a result,
 $\| x_c^{(l)} \|  \le \|x_c^{(0)}\| = \|x^{(0)}\|$.
\end{proposition}

{\bf Insensitivity to input deformation.} 
We consider the deformation of the input ``module" to a global rotation. 
Specifically, let the deformed input be of the form
$D_\rho \circ D_\tau x^{(0)}$,
where $D_\rho$ is as in Section \ref{sec:3-1},
$\rho =\rho_{u_0,t}$ being a rigid 2D rotation,
and $D_\tau$ is a small deformation in space defined by  
\begin{equation}\label{eq:def-Dtau}
D_{\tau}x^{(0)}(u,\lambda)=x^{(0)}( u - \tau(u),\lambda),\quad \forall u \in \R^2,\, \lambda\in [ M_0],
\end{equation}
with $\tau:\mathbb{R}^{2}\to\mathbb{R}^{2}$ is $C^{2}$.
Following \cite{qiu2018dcfnet},
we assume the small distortion condition, which is 
\begin{itemize}
\item[ ]
{\bf (A3)} Small distortion: $|\nabla\tau|_{\infty} = \sup_{u} \| \nabla \tau(u) \| <\frac{1}{5}$,
with
$\|\cdot\|$ being the operator norm.
\end{itemize}
The mapping $u \mapsto u-\tau(u)$ is locally invertible,
and the constant $\frac{1}{5}$ is chosen for convenience. 
$T_\rho$ is defined in \eqref{eq:def-Trho-l}, and
the stability result is summarized as 

\begin{theorem}\label{thm:deform1}
Let $\rho = \rho_{u_0,t}$ be an arbitrary rotation in $\R^2$, around $u_0$ by angle $t$, 
and let $D_\tau$ be a small deformation. 
In a RotDCF Net, under (A1), (A2), (A3),  $c_1=4$, $c_2=2$, for any $L$,
\[
\|x^{(L)}[D_\rho\circ D_\tau x^{(0)}]  - T_\rho x^{(L)}[ x^{(0)} ] \| 
\le (  2c_1 L |\nabla \tau|_\infty + c_2 2^{-j_L} |\tau|_\infty ) \|x^{(0)}\|.
\]
\end{theorem}

Unlike previous stability results for regular CNNs, 
the above result allows an arbitrary global rotation $\rho$
with respect to which the RotDCF representation is equivariant,
apart from a small ``residual" distortion $\tau$ whose influence can be bounded. 
This is also an important result in practice, 
because most often in recognition tasks the image rotation is not a rigid in-plane one, 
but is induced by the rotation of the object in 3D space. 
Thus the actual transformation of the image may be close to a 2D rotation but is not exact. 
The above result guarantees that in such cases the RotDCF representation undergoes approximately 
an equivariant action of $T_\rho$, which implies consistency of the learned deep features up to a rotation.
The improved stability of RotDCF Net over regular CNNs 
in this situation is observed experimentally in Section \ref{sec:4}. 

To prove Theorem \ref{thm:deform1}, we firstly establish the following approximate equivariant relation for all layers $l$, 
which can be of independent interest, e.g. for estimating the image transformations. All the proofs are left to Appendix.
\begin{proposition}\label{prop:deform2}
In a RotDCF Net, under (A1), (A2), (A3), $c_1=4$, for any $l$,
\[
\| x^{(l)}[D_\rho\circ D_\tau x^{(0)}]  - T_\rho \circ D_\tau x^{(l)}[ x^{(0)} ] \| 
\le  2c_1 l |\nabla \tau|_\infty  \|x^{(0)}\|,
\]
where $D_\tau$ only acts on the space variable $u$ of  $x^{(l)}$ similar to \eqref{eq:def-Dtau}.
\end{proposition}

%%%%%%%%%%%%%
%%
\section{Experimental Results}\label{sec:4}

In this section, we experimentally test the performance of RotDCF Nets on object classification and face recognition tasks. 
The advantage of RotDCF Net is demonstrated via improved recognition accuracy and robustness to rotations of the object,
not only with in-plain rotations but with 3D rotations as well.
To illustrate the rotation equivariance of the RotDCF deep features, 
we show that a trained auto-encoder with RotDCF encoder layers 
is able to reconstruct rotated digit images from ``circulated" codes. 
All codes will be publicly available.

\begin{table}[b]
\vskip -0.15 in
\scriptsize
\parbox{0.5 \linewidth}{
\begin{centering}
\begin{tabular}{ l | c   | c c  }
 \hline
 \multicolumn{4}{c}{rotMNIST   Conv-3, $N_{\text{tr}}=10K$ }   \\
\hline
                         			&  			Test Acc.	  									&    	\# param.	&  Ratio  \\
\hline
CNN  $M$=32			&  	95.67											& 	2.570$\times 10^5$		&	1.00	\\
DCF  $M$=32, $K$=5	& 	95.58											& 	5.158$\times 10^4$		&	0.20	\\
DCF   $M$=32, $K$=3	& 	95.69											& 	3.104$\times 10^4$	 	&	0.12	\\
\hline
RotDCF  $N_\theta$ = 8  & 													& 					 & 	  \\
$M$=16, $K$=14, $K_\alpha$=8    & 		97.86								& 	2.871$\times 10^5$		 & 	1.12    \\
$M$=16, $K$=5, $K_\alpha$=8 & 			97.81								& 	1.026$\times 10^5$  		 & 	0.40  \\
$M$=16, $K$=3, $K_\alpha$=8 & 			97.77								& 	6.160$\times 10^4$  		 & 	0.24  \\
$M$=16, $K$=5, $K_\alpha$=5 & 			97.96								& 	6.419$\times 10^4$  		 & 	0.25  \\
$M$=16, $K$=3, $K_\alpha$=5 & 			97.95								& 	3.856$\times 10^4$ 		 & 	0.15  \\
$M$=8, $K$=5, $K_\alpha$=5 & 			97.81								& 	1.610$\times 10^4$ 		 & 	0.06  \\
$M$=8, $K$=3, $K_\alpha$=5 & 			97.59								& 	9.680$\times 10^3$   	 & 	0.04  \\
 \hline
 \end{tabular}
 \end{centering}
 }
 ~
 \parbox{0.45 \linewidth}{
 \begin{centering}
 \begin{tabular}{ l | c   | c c  }
  \hline
  \multicolumn{4}{c}{rotMNIST   Conv-3, $N_{\text{tr}}=5K$ }   \\
\hline
                         &  	Test Acc.	  												&    	\# param.	&  Ratio \\
\hline
CNN  $M$=32			&  	94.04											& 				 	&		\\
DCF   $M$=32, $K$=3	& 	94.08											& 				 	&		\\
\hline
RotDCF  $N_\theta$=8  & 													& 			   		 & 	  \\
$M$=16, $K$=3, $K_\alpha$=5 & 			96.79								& \multicolumn{2}{c}{(same as left)}   	  \\
$M$=8, $K$=3, $K_\alpha$=5 & 			96.53								& 			   		 & 	  \\
\hline
\hline
     \multicolumn{4}{c}{CIFAR10  VGG-16, $N_{\text{tr}}=10K$ }   \\
\hline
\hline
CNN	 $M=64$ 	&  		78.40										&  2.732$\times 10^6$		&	1.00	\\
\hline
RotDCF,  $N_\theta$= 8  & 											& 		   		 & 	  	\\
$M$=32, $K$=3, $K_\alpha$=7 & 			79.44 						&   1.593$\times 10^6$		 &   0.58  	\\
$M$=32, $K$=3, $K_\alpha$=5 & 			79.53						&   1.138$\times 10^6$   	 	&   0.42  	\\      
\hline
\end{tabular}
\end{centering}
}
\caption{
\label{tab:acc-1}
\small
Classification accuracy using regular CNN, DCF and RotDCF Nets on {\bf rotMNIST} and {\bf CIFAR10}. 
``\# param." is number of parameters in all convolutional layers,
and ``Ratio" indicates the proportion to the \# param. of the regular CNN.
Notice that the reduction from non-bases rotation-equivariant CNNs (the fair comparison case) can be even smaller, 
which is the factor of $\frac{KK_\alpha}{L^2 N_\theta}$, c.f. Section \ref{sec:2-3}. 
The higher accuracy is due to the group equivariance and the lower model complexity is due to the bases decomposition.
}
\vskip -0.3in
\end{table}

\subsection{Object Classification}

\begin{table}[h]
\begin{centering}
\small
\begin{tabular}[t]{  l | l }
%\hline
\hline
~~ Conv-3 CNN-$M$ 						 &   	Conv-3 RotDCF-$M$ ~~          \\
\hline
c5x5x1x$M$ ReLu ap2x2  			&      rc5x5x1x$M$  ReLu  ap2x2\\
c5x5x$M$x$2M$  ReLu  ap2x2 		& 	 rc5x5x$N_\theta$x$M$x$2M$   ReLu  ap2x2\\
c5x5x$2M$x$4M$  ReLu  ap2x2 		& 	 rc5x5x$N_\theta$x$2M$x$4M$  ReLu  ap2x2\\
  fc64 ReLu  fc10 softmax-loss			& 	  fc64 ReLu  fc10 softmax-loss \\
%\hline			
\hline       	
\end{tabular}
\caption{\label{tab:network-arch-1}
\small
Conv-3 network architectures used in rotMNIST, $M=32$, 16 or 8. 
c$L$x$L$x$M'$x$M$ stands for a convolutional layer of patch size $L$x$L$ and input (output) channel $M'$ ($M$).
ap$L$x$L$ stands for $L$x$L$ average-pooling. 
In the RotDCF Net,
rc$L$x$L$x$N_\theta$x$M'$x$M$ stands for a rotation-indexed convolutional layer,
which includes $N_\theta$-times many number of filters except for the 1st rc layer (see Section \ref{sec:2}). 
Batch-normalization layers (not shown) are used during training. 
}
\end{centering}
\vskip -0.1in
\end{table}

\begin{table}[b]
\begin{centering}
%\vskip -0.05in
\scriptsize
\begin{tabular}{  l | l }
%\hline
\hline
~~ VGG-16 CNN-$M$ 						 							&   	VGG-16 RotDCF-$M$ ~~          \\
\hline
c3x3x3x$M$ ReLu  c3x3x$M$x$M$ ReLu  c3x3x$M$x$M$ ReLu 		&      rc3x3x3x$M$ ReLu  rc3x3x$N_\theta$x$M$x$M$ ReLu  rc3x3x$N_\theta$x$M$x$M$ ReLu \\
c3x3x$M$x$M$ ReLu  c3x3x$M$x$M$ ReLu   mp2x2  				&      rc3x3x$N_\theta$x$M$x$M$ ReLu  rc3x3x$N_\theta$x$M$x$M$ ReLu   mp2x2\\
c3x3x$M$x$2M$ ReLu  c3x3x$2M$x$2M$ ReLu 				 	&      rc3x3x$N_\theta$x$M$x$2M$ ReLu  rc3x3x$N_\theta$x$2M$x$2M$ ReLu  				\\
c3x3x$2M$x$2M$ ReLu  c3x3x$2M$x$2M$ ReLu   mp2x2  			&      rc3x3x$N_\theta$x$2M$x$2M$ ReLu  rc3x3x$N_\theta$x$2M$x$2M$ ReLu   mp2x2\\
c3x3x$2M$x$4M$ ReLu  c3x3x$4M$x$4M$ ReLu 				 	&      rc3x3x$N_\theta$x$2M$x$4M$ ReLu  rc3x3x$N_\theta$x$4M$x$4M$ ReLu  				\\
c3x3x$4M$x$4M$ ReLu  c3x3x$4M$x$4M$ ReLu   mp2x2  			&      rc3x3x$N_\theta$x$4M$x$4M$ ReLu  rc3x3x$N_\theta$x$4M$x$4M$ ReLu   mp2x2\\
 fc128 ReLu  fc10 softmax-loss									& 	 fc128 ReLu  fc10 softmax-loss \\
%\hline			
\hline       	
\end{tabular}
\caption{\label{tab:network-arch-2}
VGG-16-like network architectures used in CIFAR10, $M=64$ or 32. 
mp$L$x$L$ stands for $L$x$L$ max-pooling,
and other notations similar to Table \ref{tab:network-arch-1}.
}
\end{centering}
\vskip -0.15in
\end{table}

{\bf  Non-transfer learning setting}. 
The
{\bf rotMNIST} 
dataset contains $28 \times 28$ grayscale images of digits from 0 to 9, 
randomly rotated by an angle uniformly distributed from 0 to $2\pi$ \cite{Cohen2016}.
We use 10,000 and 5,000 training samples, and 50,000 testing samples. 
A CNN consisting of 3 convolutional layers (Conv-3, Table \ref{tab:network-arch-1}) 
is trained as a performance baseline,
and the RotDCF counterpart is made by replacing the regular convolutional layers with RotDCF layers, 
with reduced number of (unstructured) channels $M$, and $N_\theta$ many rotation-indexed channels ($N_\theta = 8$). 
$K$ bases are used for $\psi_k$ and $K_\alpha$ for $\varphi_m$.
The classification accuracy is shown in Table \ref{tab:acc-1}, for various choices of $M$, $K$ and $K_\alpha$.  
We see that RotDCF net obtains improved classification accuracy with significantly reduced number of parameters, e.g.,
with 10K training, the smallest RotDCF Net ($M=8$, $K=3$, $K_\alpha=5$)
improves the test accuracy from $95.67$ to $97.59$ with less than $\frac{1}{20}$ many parameters of the CNN model,
and $\frac{1}{3}$ of the DCF model \cite{qiu2018dcfnet}.
The trend continues with reduced training size (5K). 

The {\bf CIFAR10} dataset consists of $32 \times 32$ colored images from 10 object classes \cite{cifar},
and we use 10,000  training and 50,000 testing samples. 
The network architecture is modified from VGG-16 net \cite{simonyan2014very} (Table \ref{tab:network-arch-2}).
As shown in  Table \ref{tab:acc-1}, 
RotDCF Net obtains better testing accuracy with reduced model size from the regular CNN baseline model. 

{\bf Transfer learning setting.}
We train a regular CNN and a RotDCF Net on 10,000 up-right MNIST data samples,
and directly test on 50,000 randomly rotated MNIST samples
where the maximum rotation angle $\text{MaxRot}$=30 or 60 degrees
(the ``no-retrain" case).
We also test after retraining the last two non-convolutional layers (the `` fc-retrain" case).
To visualize the importance of image regions which contribute to the classification accuracy, 
we adopt Class Activation Maps (CAM) \cite{CAM},
and the network is modified accordingly by removing the last pooling layer in the net in Table \ref{tab:network-arch-1}
and inserting a ``gap" global averaging layer.
The test accuracy are listed in Table \ref{tab:acc-transfer-mnist},
where the superiority of RotDCF Net is clearly shown in both the ``no-retrain"  and `` fc-retrain"  cases .
The improved robustness of RotDCF Net is furtherly revealed by the CAM maps (Figure \ref{fig:transfer-mnist-CAM}):
the red-colored region is more stable for RotDCF Net even in the case with retraining.

\begin{figure}[t]
%\vskip -0.2in
\begin{center}
\includegraphics[width=.495\linewidth]{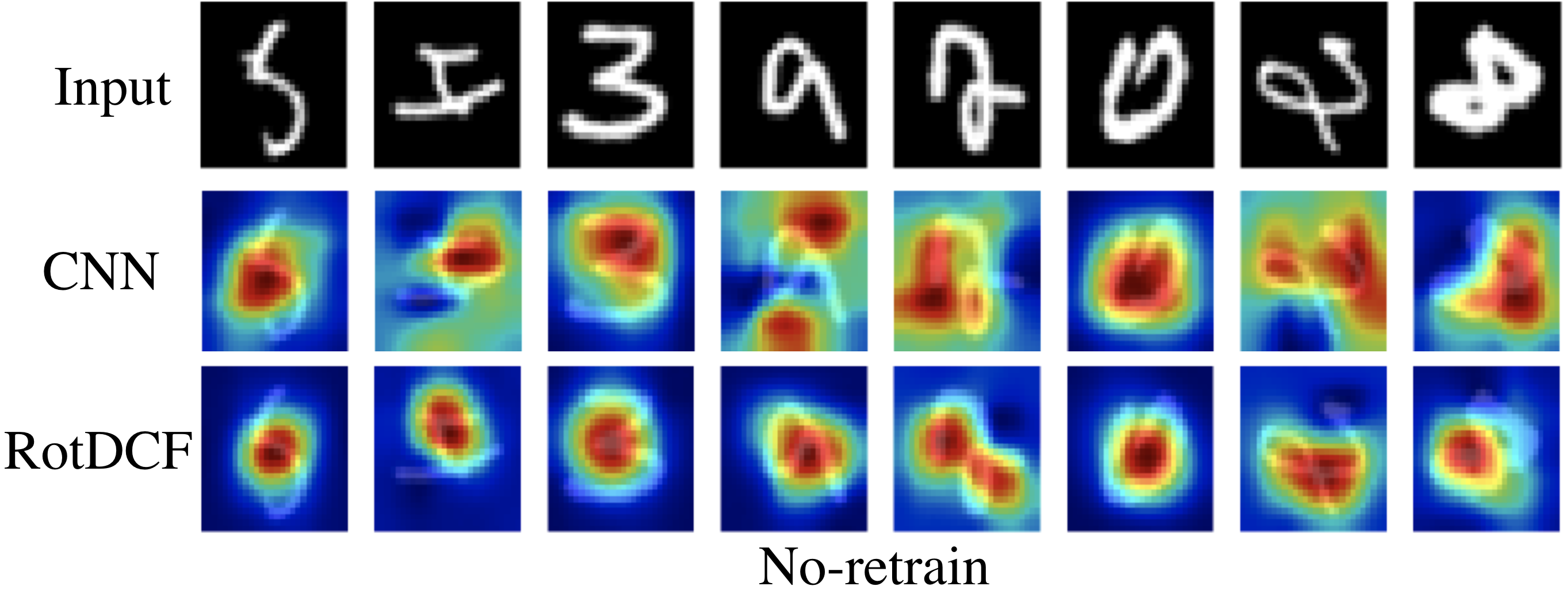} 
\includegraphics[width=.495\linewidth]{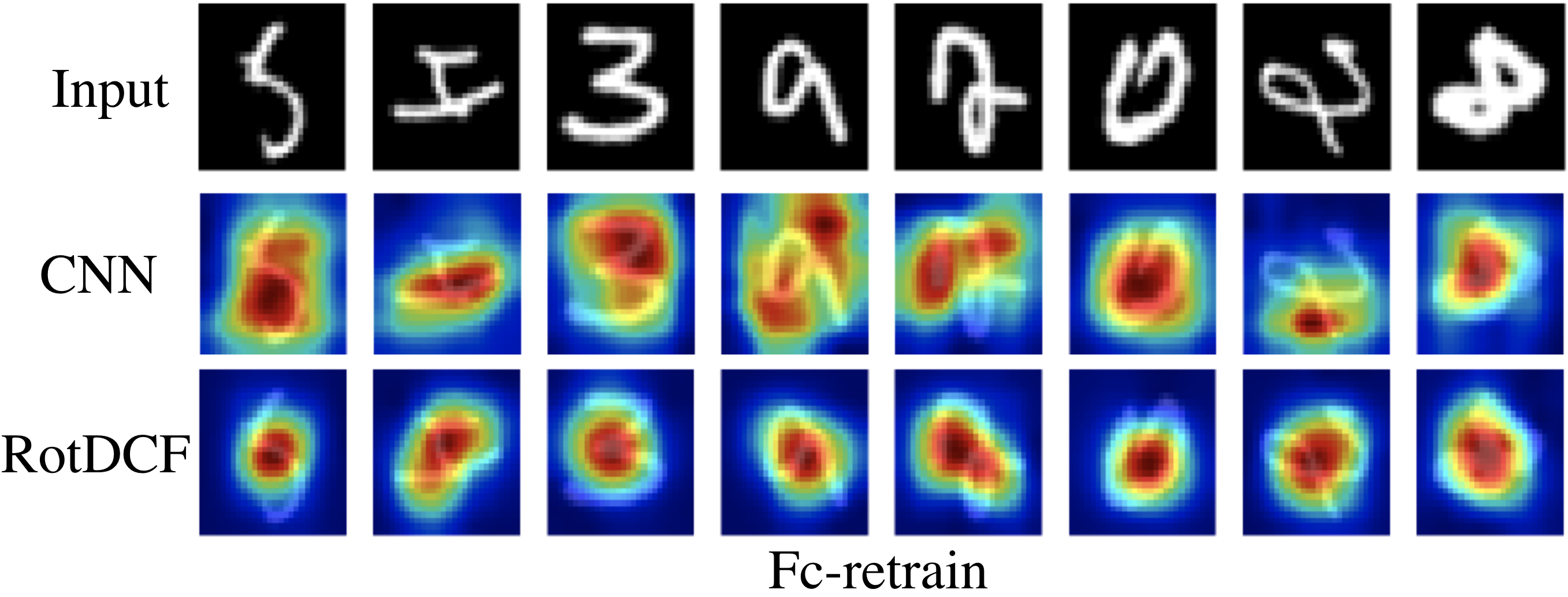} 
\caption{\small 
Representative
class activation maps (CAM) on testing images in the rotMNIST transfer learning experiment.
The heatmap indicates the importance of image regions used in recognizing a digit class. 
The CNN and RotDCF networks are trained on up-right MNIST samples,
with no retrain and retraining the fully connected layers respectively before testing,
and the testing samples are randomly rotated up to 60 degrees. (c.f. Table \ref{tab:acc-transfer-mnist}).
}
\vskip 0.1 in
\label{fig:transfer-mnist-CAM}
\end{center}
\end{figure}

\begin{table}
\scriptsize
%\small
\begin{center}
 \parbox{0.45 \linewidth}{
 \centering
\begin{tabular}[t]{  l  |   c c  }
\hline
     \multicolumn{3}{c}{ MNIST to rotMNIST  \text{MaxRot}=30 Degrees}   \\
\hline
                               &  	 no-retrain						        &    fc-retrain  \\
\hline
CNN 	                &   	92.61							& 	94.71			 			\\
RotDCF 			& 	96.90							& 	98.48			 			\\
\hline
\end{tabular}
} ~~
 \parbox{0.45 \linewidth}{
 \centering
\begin{tabular}[t]{  l  |   c c  }
\hline
  \multicolumn{3}{c}{ MNIST to rotMNIST \text{MaxRot}=60 Degrees}   \\
\hline
                               &  	 no-retrain						        &    fc-retrain  \\
\hline
CNN 	                &   	69.61							& 	85.90			 			\\
RotDCF 			& 	82.36							& 	97.68			 			\\
\hline
\end{tabular}
}
\caption{\label{tab:acc-transfer-mnist}
\small
Test accuracy in the rotMNIST transfer learning experiment.
The network is trained on 10K up-right MNIST samples and tested on 
50K randomly rotated samples up to the MaxRot degrees.
}
\end{center}
\vskip -0.05 in
\end{table}

\subsection{Image Reconstruction}

\begin{table}
\vskip -0.05 in
\begin{centering}
\small
\begin{tabular}{  l  }
%\hline
\hline
~~ RotDCF ConvAE 						 		 \\
\hline
rc5x5x1x8 ReLu  ap2x2  							\\
rc5x5x$N_\theta$x8x16  ReLu   ap2x2				\\
rc5x5x$N_\theta$x16x32 ReLu  ap2x2			 	\\
rc5x5x$N_\theta$x32x32  ReLu  ~~~~~~~~~ $\leftarrow$ {\bf Encoded representation}		\\
 fc128 ReLu
ct5x5x128x16$N_\theta$ ReLu 					\\
ct5x5x16$N_\theta$x8$N_\theta$ (upsample 2x2) ReLu 	  			\\
ct5x5x8$N_\theta$x1 (upsample 2x2)   Eucledian-loss 				\\		
\hline       	
\end{tabular}
\caption{\label{tab:network-arch-3}
Convolutional Auto-encoder network used in the image reconstruction experiment. 
RotDCF layers are used in the encoder network, with $N_\theta =16$, $K=5$, $K_\alpha= 5$ ($K=8$, $K_\alpha= 15$ in the last RotDCF layer),
and transposed-convolutional layers (denoted by ``ct") with upsampling  are used in the decoder net. 
ap$L$x$L$ stands for $L$x$L$ average-pooling,
and other notations similar to Table \ref{tab:network-arch-1}.
}
\end{centering}
\vskip -0.2in
\end{table}

\begin{figure}
%\vskip -0.2in
\begin{center}
 \includegraphics[width=1\linewidth]{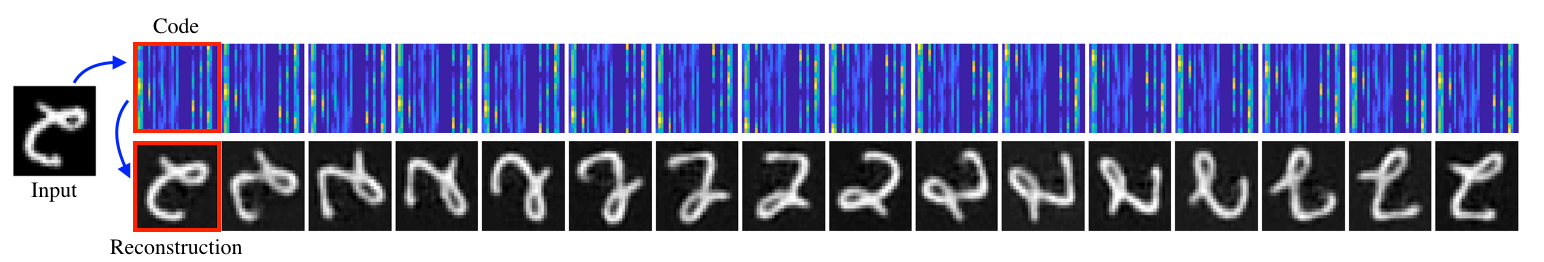} 
\caption{
\small 
Codes and reconstructions of rotMNIST digits. 
(Top) A test image is encoded into a $16 \times 32$ array in the red box (the intermediate representation), 
and the code generates $16$ copies by circulating the rows.
(Bottom) 
Images reconstructed from the row-circulated codes above by the decoder.
}
\label{fig:ae1}
\end{center}
\vskip -0.1 in
\end{figure}

To illustrate the explicit encoding of group actions in the RotDCF Net features, 
we train a convolutional auto-encoder on the rotated MNIST dataset,  
where encoder consists of stacked RotDCF layers, 
and the decoder consists of stacked transposed-convolutional layers (Table \ref{tab:network-arch-3}).
The encoder maps a 28$\times$28 image into an array of $16 \times 32$, 
where the first dimension is the discretization of the rotation angles in $[0,2\pi]$,
and the second dimension is the unstructured channels. 
Due to the rotation equivariant relation, 
the ``circulation'' of the rows of the code array should correspond to the rotation of the image.
This is verified in Figure \ref{fig:ae1}:
The top panel shows the code array produced from a testing image, and the 16 row-circulated copies of it.
The bottom panel shows the output of the decoder fed with the codes in the top panel.

\subsection{Face Recognition}\label{sec:4.3}
  
  \begin{table}[b]
\begin{centering}
\scriptsize
\begin{tabular}{  l | l }
\hline
~~CNN   & ~~ RotDCF 	  \\
\hline
c5x5x3x32 ReLu mp2x2  							& rc5x5x3x16 ReLu mp2x2\\
c5x5x32x64  ReLu  mp2x2 						& rc5x5x$N_\theta$x16x32  ReLu  mp2x2 \\
c5x5x64x128  ReLu  c5x5x128x128  ReLu mp2x2 		& rc5x5x$N_\theta$x32x64  ReLu  c5x5x64x64  ReLu mp2x2 \\
c5x5x128x256  ReLu  c5x5x256x256  ReLu mp2x2 		& rc5x5x$N_\theta$x64x128  ReLu  c5x5x128x128  ReLu mp2x2 \\
c5x5x256x256  ReLu  c5x5x256x256  ReLu gap13x13 	& rc5x5x$N_\theta$x128x128  ReLu  c5x5x128x128  ReLu gap13x13 \\
fc softmax 									& fc softmax\\
%\hline			
\hline       	
\end{tabular}
\caption{\label{tab:facenet}
Network architectures used in face experiments,
notations as in Table  \ref{tab:network-arch-1}. 
In the RotDCF Net, $N_{\theta} =8$, $K=5$, $K_{\alpha} = 5$. 
}
\end{centering}
%\vskip -0.2in
\end{table}

As a real-world example, we test RotDCF on  the Facescrub dataset \cite{ng2014data} containing over 100,000 face images of 530 people.
A CNN and a RotDCF Net (Table \ref{tab:facenet})
are trained respectively using the gallery images from the 500 known subjects,
which are preprocessed to be near-frontal and upright-only by aligning facial landmarks \cite{landmark}.
 See Appendix \ref{sec:A3-face} for data preparation and training details. 
For the trained deep networks, we remove the last \emph{softmax} layer, and then use the network outputs as deep features for faces, 
which is the typical way of using deep models for face verification and recognition to support both seen and unseen subjects \cite{vgg-face}.
Using deep features generated by the trained networks,
a probe image is then compared with the gallery faces whose identities assume known, 
and classified as the identity label of the top match.

Under this gallery-probe face recognition setup, we obtain 94.10\% and 96.92\% accuracy for known and unknown subjects respectively
using the CNN model; using RotDCF, the accuracies are  93.42\% and 96.92\%. 
Testing on unknown subjects are critical for validating the model representation power over unseen identities,
and the reason for higher accuracy is simply due to the smaller number of classes. 
For both cases, RotDCF reports comparable performance as CNN,
while the number of parameters in the RotDCF model is about one-fourth of the CNN model (see 
Appendix \ref{sec:A3-face}).
 
\textbf{In-plane rotation.} 
This experiment demonstrates the rotation-equivariance of the RotDCF features. 
We apply in-plane rotations at intervals of $\frac{\pi}{4}$ to the probe images (Figure~\ref{fig:grot}), 
and let the original probe set be the new gallery, the rotated copies be the new probe set. 
In this setting, using the RotDCF model we obtain  97.04\% and 97.58\% recognition accuracy for known and unknown subjects respectively, 
after aligning the deep features by circular shifts (using the largest-magnitude $\alpha$ channel as reference).
Notice that the model only sees upright faces.
This is due to the rotation-equivariant property of the RotDCF Net,
which means that the face representation is consistent regardless of its orientations after the group alignment.
Lacking such properties, CNN obtains 0.54\% and 5.05\% recognition accuracies, which is close to random guess.
We further compare CNN and RotDCF models via the CAM maps. 
As shown in Figure~\ref{fig:grot}, 
RotDCF is able to choose more consistent regions in describing a subject in different rotated copies, 
while CNN tends to pick different regions in defining a subject.

  \begin{figure}[t]
%    \vskip -0.15 in
  \centering
{\includegraphics[angle=0, width=.492\textwidth]{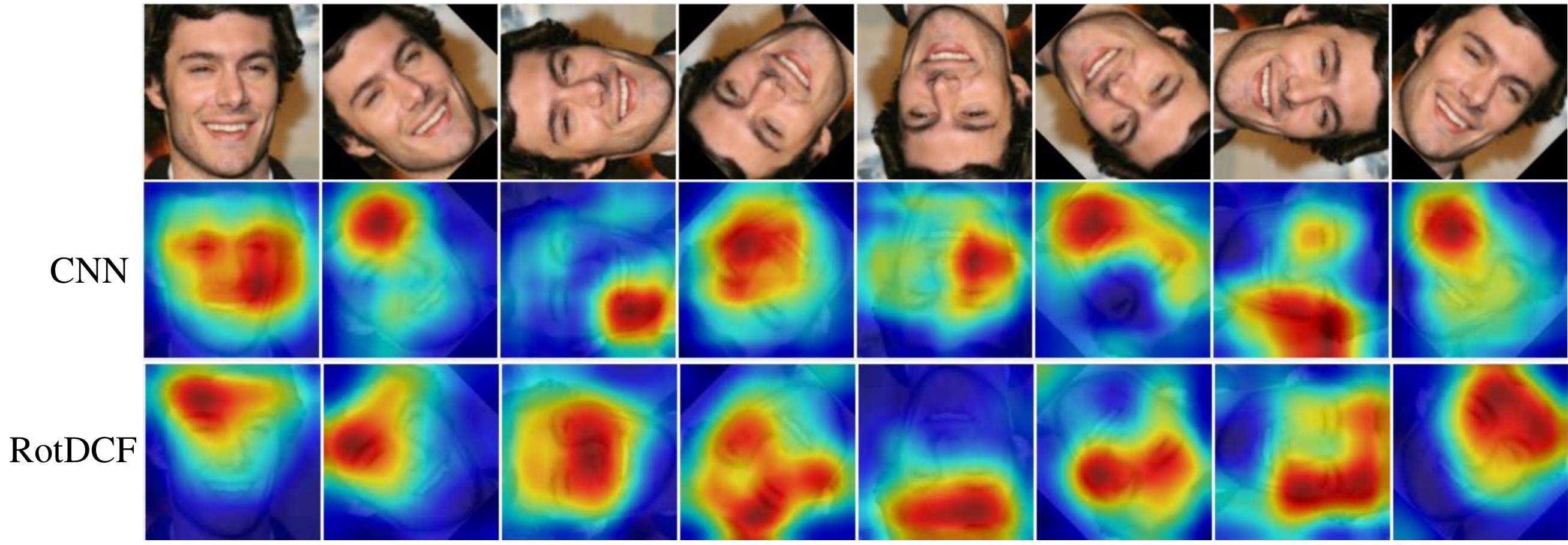} }
{\includegraphics[angle=0, width=.492\textwidth]{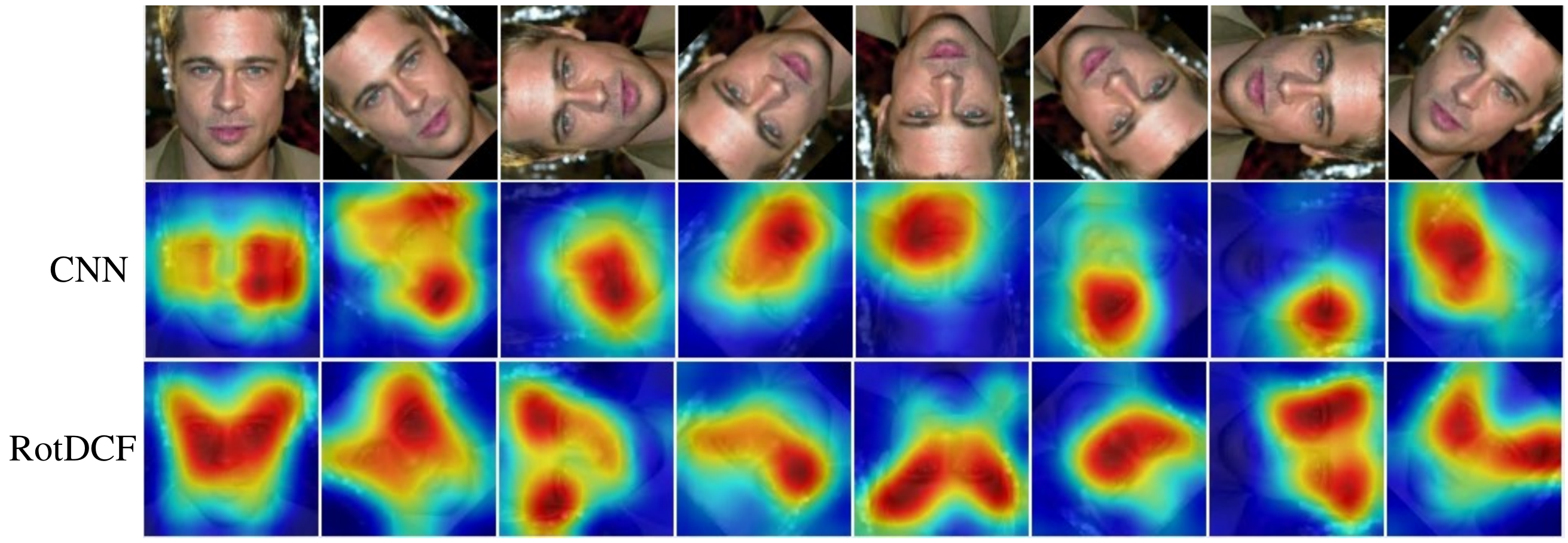} }
  \caption{
  \small
  Example CAM maps for recognizing faces with in-plane rotations.  
The heatmap indicates the importance of different image regions used by respective models in defining a face. 
  Across different in-plane rotated copies,
  RotDCF chooses significantly more consistent discriminative regions than CNN,
  indicating more stable representations. 
   In this experiment, we obtain 0.54\% recognition accuracy using CNN (nearly random guess), and  97.04\% accuracy using RotDCF with feature alignment, on known subjects.
  }
%  \vskip -0.2 in
  \label{fig:grot}
  \end{figure}

 \textbf{Out-of-plane rotation.}
 To validate our theoretical result on representation stability under input deformations, we introduce out-of-plane rotations to the probe.
Each probe image is fitted to a 3D face mesh, and rotated copies are rendered at the 10$^o$ intervals with -40$^o$ to 40$^o$ yaw, and  -20$^o$ to 20$^o$ pitch,
generating 45 synthesized faces in total (Figure~\ref{fig:lrotface}). 
The synthesis faces at two poses (highlighted in red) are used as the new gallery, and all remaining synthesis faces form the new probe.
The out-of-plane rotations here can be viewed as mild in-plane rotations plus additional variations,
a situation frequently encountered in the real world.
 With this gallery-probe setup, the RotDCF model  obtains  89.66\% and 97.01\% recognition accuracy for known and unknown subjects,
 and the accuracies are   80.79\% and 89.97\% with CNN. 
The CAM plots in Figure ~\ref{fig:lrot} also indicate that 
RotDCF Net chooses more consistent regions over CNN in describing a subject across different poses.
Since the out-of-plane rotations as in Figure~\ref{fig:lrotface} can be considered as in-plane rotations with additional variations, 
the superior performance of RotDCF is consistent with the theory in Section~\ref{sec:3}.

\section{Conclusion and Discussion}\label{sec:5}

This work introduces a decomposition of the filters in rotation-equivariant CNNs 
under joint steerable bases over space and rotations simultaneously,
obtaining equivariant deep representations with significantly reduced model size and an implicit filter regularization.
The group equivariant property and representation stability are proved theoretically.
In experiments, RotDCF demonstrates improved recognition accuracy, particularly in the transfer learning setting,
as well as better feature interpretability and stability over regular CNNs on synthetic and real-world datasets involving object rotations.
It is important to build deep networks which encode group actions and at the same time are resilient to input variations,
and RotDCF provides a general approach to achieve these two objectives.

 \begin{figure} [t]
%    \vskip -0.2 in
   \centering
   {\includegraphics[angle=0, width=.98\textwidth]{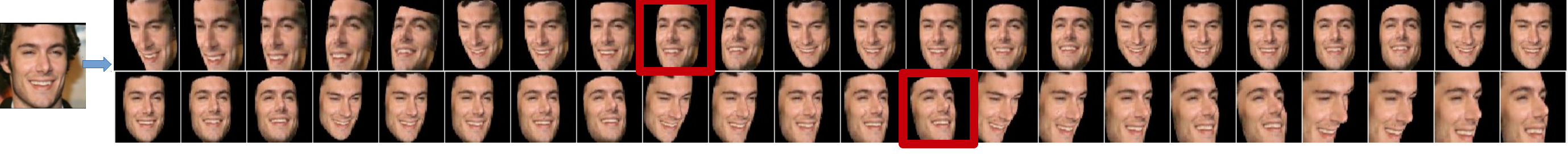} \hspace{0pt}}
   \caption{
    \small
   Synthesized faces from a testing image with -40$^o$ to 40$^o$ yaw, and -20$^o$ to 20$^o$ pitch, at a 10$^o$ interval.  
   }
%     \vskip -0.1 in
   \label{fig:lrotface}
   %\end{center}
   \end{figure}

 \begin{figure} [t]
 \centering
\includegraphics[angle=0, width=.495\textwidth]{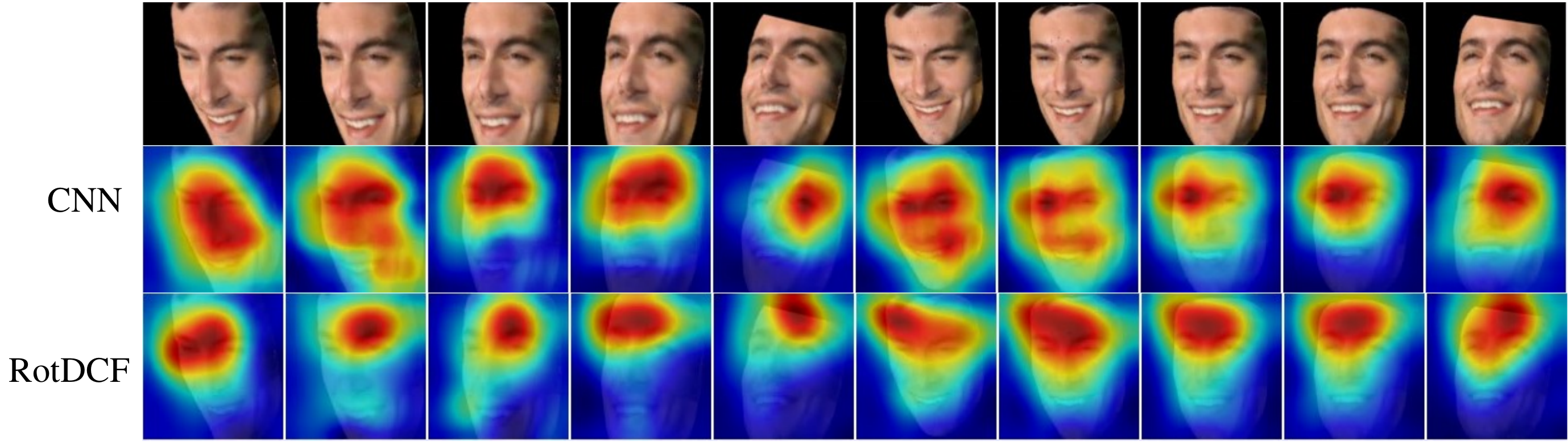} 
\label{fig:lrot2} \includegraphics[angle=0, width=.495\textwidth]{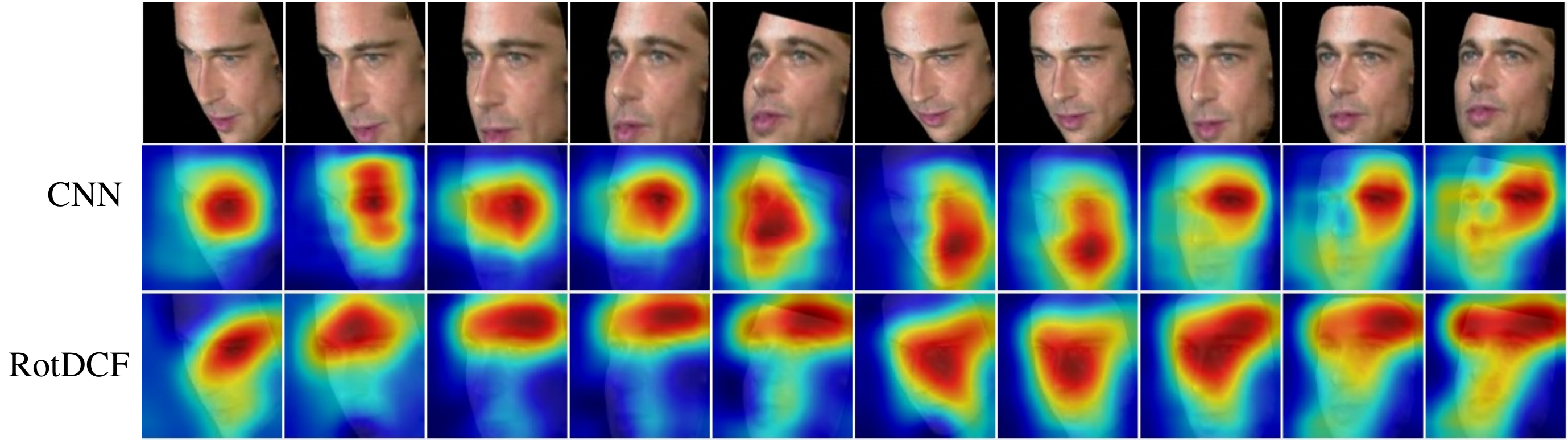} 
 \caption{
 \small
   Example CAM maps for recognizing faces with out-of-plane rotations.
    Across out-of-plane rotated copies, 
    the discriminative regions chosen by RotDCF in describing a subject are more consistent,
    showing better representation stability than CNN.
    In this experiment, we obtain 80.79\% recognition accuracy using CNN, and  89.66\% using RotDCF, on known subjects. 
 }
%   \vskip -0.2 in
 \label{fig:lrot}
 \end{figure}

To extend the work, 
implementation issues  like parallelism efficiency should be considered 
before the theoretical flop savings can be achieved. 
The memory need in the current RotDCF Net is the same as without using bases, 
because the output in each layer is computed in the real space to apply the ReLU.
It would be appealing to completely avoid the real-space representation 
so as to save memory as well as to further reduce the computation.
Finally, the proposed framework should extend to other groups due to the generality of the approach.
The application domain will govern the choice of bases $\psi_k$ and $\varphi_m$,
e.g., for $SO(3)$ the spherical harmonics would be a natural choice for $\varphi_m$.

%%%

%\small

%
%
%
%

%\newpage

\appendix

\setcounter{table}{0}
\renewcommand{\thetable}{A.\arabic{table}}

\section{Proofs in Section 3}

\begin{proof}[Proof of Theorem \ref{thm:equiv}]
Since the bases expansion under $\psi_{j_l,k}$ and $\varphi_m$ does not affect the form of convolutional layers, but only impose regularity of the filters, 
it suffices proving the statement without expanding the filters under the bases.

Observe that \eqref{eq:so2-equiv} is equivalent to that 
\begin{equation}
\label{eq:so2-equiv-2}
T_\rho x^{(l)}[x^{(l-1)}] = x^{(l)}[ T_\rho  x^{(l-1)}]
\end{equation}
for all $l$, where $T_\rho x^{(0)}$ means $D_\rho x^{(0)}$. 

The sufficiency part:
When $l=1$, by \eqref{eq:joint-conv-1}, 
\[
\begin{split}
T_\rho x^{(1)}[x^{(0)}](u, \alpha, \lambda)  
& =  \sigma( 
\sum_{\lambda'} \int_{\R^2} x^{(0)}( \rho_{u_0,t} u+v,\lambda') W^{(1)}_{\lambda', \lambda}(\Theta_{\alpha -t} v) dv + b^{(1)}(\lambda)),   \\
x^{(1)}[ D_\rho  x^{(0)}](u, \alpha, \lambda)
& =  \sigma( 
\sum_{\lambda'} \int_{\R^2} x^{(0)}( \rho_{u_0,t}(u+v),\lambda') W^{(1)}_{\lambda', \lambda}(\Theta_{\alpha} v)  dv + b^{(1)}(\lambda)).
\end{split}
\]
Since $\rho_{u_0,t}(u+v) = \rho_{u_0,t}+ \Theta_t v$, we have that $T_\rho x^{(1)}[x^{(0)}] = x^{(1)}[ D_\rho x^{(0)}]$ by a change of variable of $\Theta_t v \mapsto v$. 

When $l > 1$, by \eqref{eq:joint-conv-2}, 
\[
\begin{split}
T_\rho x^{(l)}[x^{(l-1)}](u, \alpha, \lambda)  
& =  \sigma( 
\sum_{\lambda'} \int_{\R^2} \int_{S^1} x^{(l-1)}( \rho_{u_0,t} u+v, \alpha', \lambda') W^{(l)}_{\lambda', \lambda}(\Theta_{\alpha -t} v, \alpha'-\alpha+t) dv d\alpha' + b^{(l)}(\lambda)),   \\
x^{(l)}[ T_\rho  x^{(l-1)}](u, \alpha, \lambda)
& =  \sigma( 
\sum_{\lambda'} \int_{\R^2}  \int_{S^1} x^{(l-1)}( \rho_{u_0,t}(u+v),\alpha'-t, \lambda') W^{(l)}_{\lambda', \lambda}(\Theta_{\alpha} v, \alpha'-\alpha)  dv d\alpha' + b^{(l)}(\lambda)).
\end{split}
\]
Again, inserting $\rho_{u_0,t}(u+v) = \rho_{u_0,t}+ \Theta_t v$, 
the claim follows by changing variables $\Theta_t v \mapsto v$ and $\alpha' -t \mapsto \alpha'$.

The necessity part: 
When $l=1$, 
denote the general convolutional filter as $w^{(1)}(v; \lambda', \lambda, \alpha)$, and 
\[
x^{(1)}(u, \alpha, \lambda) 
= \sigma( 
\sum_{\lambda'} \int_{\R^2} x^{(0)}(u+v,\lambda') w^{(1)}(v; \lambda', \lambda, \alpha) dv + b^{(1)}(\lambda)
).
\]
Recall that 
\[
\begin{split}
T_\rho x^{(1)}[x^{(0)}](u, \alpha, \lambda)  
& =  \sigma( 
\sum_{\lambda'} \int_{\R^2} x^{(0)}( \rho_{u_0,t} u+v,\lambda') w^{(1)}(v; \lambda', \lambda, \alpha-t) dv + b^{(1)}(\lambda)),   \\
x^{(1)}[ D_\rho  x^{(0)}](u, \alpha, \lambda)
& =  \sigma( 
\sum_{\lambda'} \int_{\R^2} x^{(0)}( \rho_{u_0,t}(u+v),\lambda') w^{(1)}(v; \lambda', \lambda, \alpha) dv + b^{(1)}(\lambda)) 
\end{split}
\]
and then 
\eqref{eq:so2-equiv-2} with $l=1$ holding
for any $x^{(0)}$ implies that 
\[
\sum_{\lambda'} \int_{\R^2} x^{(0)}( \rho_{u_0,t} u+v,\lambda') w^{(1)}(v; \lambda', \lambda, \alpha-t) dv 
= 
\sum_{\lambda'} \int_{\R^2} x^{(0)}( \rho_{u_0,t}(u+v),\lambda') w^{(1)}(v; \lambda', \lambda, \alpha) dv.
\]
By that $\rho_{u_0,t}(u+v) = \rho_{u_0,t}+ \Theta_t v$, the above equality gives that 
\[
w^{(1)}(v; \lambda', \lambda, \alpha-t) = w^{(1)}(\Theta_t^{-1} v; \lambda', \lambda, \alpha), 
\quad \forall \alpha, t \in S^1.
\]
Let $t=\alpha$, and $F_{\lambda', \lambda}(v) = w^{(1)}(v; \lambda', \lambda, 0)$,
we have that \[
w^{(1)}(\Theta_\alpha^{-1} v; \lambda', \lambda, \alpha) = F_{\lambda', \lambda}(v),
\]
and this gives that $ w^{(1)}(v; \lambda', \lambda, \alpha)= F_{\lambda', \lambda}( \Theta_\alpha v) $.
This proves \eqref{eq:joint-conv-1}. 

When $l>1$, consider the general convolutional filter as $w^{(l)}(v; \lambda', \lambda, \alpha', \alpha)$.
Using a similar argument, \eqref{eq:so2-equiv-2} implies that 
\[
w^{(l)}(v; \lambda', \lambda, \alpha', \alpha-t) = w^{(l)}(\Theta_t^{-1} v; \lambda', \lambda, \alpha' + t, \alpha),
\quad \forall \alpha, t \in S^1.
\]
Let $t=\alpha$, and $F_{ \lambda', \lambda}(v,\alpha') = w^{(l)}(v; \lambda', \lambda, \alpha', 0)$, 
then
\[
F_{ \lambda', \lambda}(v,\alpha') = w^{(l)}(\Theta_\alpha^{-1} v; \lambda', \lambda, \alpha' + \alpha, \alpha),
\]
which gives that $w^{(l)}( v; \lambda', \lambda, \alpha', \alpha)  = F_{ \lambda', \lambda}( \Theta_\alpha v,\alpha'-\alpha) $,
which proves \eqref{eq:joint-conv-2}. 
\end{proof}

\begin{proposition}\label{prop:Bl-Cl-Dl}
For all $l$,
\[
B^{(l)}_{\lambda', \lambda} ,\, C ^{(l)}_{\lambda', \lambda}, 2^{j_l} D^{(l)}_{\lambda', \lambda}   
\le \pi \|a^{(l)}_{\lambda', \lambda}\|_{\text{FB}},
\]
where
\begin{equation}\label{eq:def-B-C-D-1}
\begin{split}
B^{(l)}_{\lambda', \lambda} 
& := \int_{\R^2} \int_{S^1}|W^{(l)}_{\lambda', \lambda} (v,\beta)|dv d\beta, \, l>1, 
\quad
B^{(1)}_{\lambda', \lambda}  := \int_{\R^2} |W^{(1)}_{\lambda', \lambda} (v)|dv \\
C^{(l)}_{\lambda', \lambda} 
& := \int_{\R^2} \int_{S^1} |v| |\nabla_v W^{(l)}_{\lambda', \lambda} (v,\beta)|dv d\beta, \, l>1, 
\quad
C^{(1)}_{\lambda', \lambda}  := \int_{\R^2} |v| |\nabla_v W^{(1)}_{\lambda', \lambda} (v)|dv \\
D^{(l)}_{\lambda', \lambda} 
& := \int_{\R^2} \int_{S^1}  |\nabla_v W^{(l)}_{\lambda', \lambda} (v,\beta)|dv d\beta, \, l>1, 
\quad
D^{(1)}_{\lambda', \lambda}  := \int_{\R^2} |\nabla_v W^{(1)}_{\lambda', \lambda} (v)|dv 
\end{split}
\end{equation}
As a result, 
\[
 B_l, \, C_l, \, 2^{j_l} D_l  
 \le A_l,
 \]
 where
\begin{equation}
\begin{split}
B_l : = 
&  \max \{ 
	 \sup_{\lambda} \sum_{\lambda'=1}^{M_{l-1}} B^{(l)}_{\lambda', \lambda}, \, 
	\sup_{\lambda'}  \frac{M_{l-1}}{M_l}  \sum_{\lambda=1}^{M_{l}} B^{(l)}_{\lambda', \lambda}  \}, \\
C_l : = 
&  \max \{ 
	 \sup_{\lambda} \sum_{\lambda'=1}^{M_{l-1}} C^{(l)}_{\lambda', \lambda} , \, 
	\sup_{\lambda'}  \frac{M_{l-1}}{M_l}  \sum_{\lambda=1}^{M_{l}}  C^{(l)}_{\lambda', \lambda}  \}, \\
D_l : = 
&  \max \{ 
	 \sup_{\lambda} \sum_{\lambda'=1}^{M_{l-1}} D^{(l)}_{\lambda', \lambda} , \, 
	\sup_{\lambda'}  \frac{M_{l-1}}{M_l}  \sum_{\lambda=1}^{M_{l}} D^{(l)}_{\lambda', \lambda}  \}, 
\end{split}	
\label{eq:def-Bl-Cl-Dl}
\end{equation}
and thus (A2) implies that $ B_l, \, C_l, \, 2^{j_l} D_l   \le 1$ for all $l$. 
\end{proposition}
\begin{proof}[Proof of Proposition \ref{prop:Bl-Cl-Dl}]
The proof for the case of $l=1$ is the same as Lemma 3.5 and Proposition 3.6 of \cite{qiu2018dcfnet}.
We reproduce it for completeness.
When $l=1$, it suffices to show that for $F(v) = \sum_k a_k \psi_k(v)$, 
\begin{equation}\label{eq:proof-prop-Bl-Cl-Dl-1}
\int |F(v)| dv, \,
\int |v| |\nabla F(v)| dv, \,
\int |\nabla F(v)| dv \le  \pi (\sum_k \mu_k a_k^2)^{1/2}.
\end{equation}
Rescaling to $\psi_{j_l, k}$ in $v$ leads to the desired inequality with the factor of $2^{j_l}$ for $D^{(l)}_{\lambda',\lambda}$. 
To prove \eqref{eq:proof-prop-Bl-Cl-Dl-1}, observe that $F$ is supported on the unit disk, and then
$\|F\|_1, \, \int |v| |\nabla F(v)| dv \le \| \nabla F \|_1 \le \sqrt{\pi} \|\nabla F\|_2$,
where $ \|\nabla F\|_2^2 = \pi \sum_k \mu_k a_k^2$ due to the orthogonality of $\psi_k$.  

For $l>1$, similarly, we only consider the rescaled filters supported on the unit disk in $v$. 
Let $F(v,\beta) = \sum_{k,m} a_{k,m} \psi_k(v) \varphi_m(\beta)$, $\beta \in S^1$,
similarly as above, we have that 
\[
 \int \int |F(v,\beta) | dv d\beta, \,  \int \int |v| |\nabla_v F(v,\beta) | dv d\beta 
\le \int \int |\nabla_v F(v,\beta) | dv d\beta
\le (\pi \int \int |\nabla_v F(v,\beta) |^2 dv d\beta)^{1/2}
\]
recalling that $\int d\beta$ on $S^1$ has the normalization of $\frac{1}{2\pi}$.
Again, $ \int \int |\nabla_v F(v,\beta) |^2 dv d\beta = \pi \sum_{k,m} \mu_k a_{k,m}^2$ due to the orthogonality of $\psi_k$ and $\varphi_m$.
This proves that \[
\int \int |\nabla_v F(v,\beta) | dv d\beta \le \pi (\sum_{k,m} \mu_k a_{k,m}^2)^{1/2},
\]
which leads to the claim after a rescaling of $v$. 
\end{proof}

\begin{remark}[Remark to Proposition \ref{prop:l2stable}]
The proposition only needs $B_l $, defined in \eqref{eq:def-Bl-Cl-Dl}, to be less than 1 for all $l$, in a rotation-equivariant CNN,
which is implied by (A2) by Proposition \ref{prop:Bl-Cl-Dl}.
\end{remark}

\begin{proof}[Proof of Proposition \ref{prop:l2stable}]
The proof is similar to that of Proposition 3.1(a) of \cite{qiu2018dcfnet}. 
Specifically, in (a), the argument is the same for $l=1$, making use of the fact that \[
\int |w(\Theta_\alpha v )|dv = \int |w( v )|dv, \quad \forall \alpha \in S^1
\]
and $\int_{S^1} d\alpha =1 $ due to the normalization of $\frac{1}{2\pi}$.
For $l>1$, the same technique proceeds with the new definition of $B^{(l)}_{\lambda', \lambda}$ as in \eqref{eq:def-B-C-D-1}
which involves the integration of $\int_{S^1} (\cdots)  d\beta$. The detail is omitted.

To prove (b),  we firstly verify that $x_0^{(l)}$ only depends on $\lambda$. When $l=1$, 
$x^{(1)}_0(u,\alpha,\lambda) = \sigma( b^{(1)} (\lambda) )$. 
Suppose that it holds for $(l-1)$, consider $l>1$, 
\[
\begin{split}
x^{(l)}_0( u, \alpha, \lambda ) 
& = 
\sigma(  \sum_{\lambda'} \int_{S^1} \int_{\R^2} x_0^{(l-1)}(u+v, \alpha +\beta, \lambda')  W^{(l)}_{\lambda',\lambda}(\Theta_\alpha v, \beta) dv d\beta 
+ b^{(l)}(\lambda) ) \\
& =
\sigma(  \sum_{\lambda'} x_0^{(l-1)}(\lambda')  \int_{S^1} \int_{\R^2}  W^{(l)}_{\lambda',\lambda}(\Theta_\alpha v, \beta) dv d\beta 
+ b^{(l)}(\lambda) ) \\
& =
\sigma(  \sum_{\lambda'} x_0^{(l-1)}(\lambda') \cdot  \int_{S^1} \int_{\R^2}  W^{(l)}_{\lambda',\lambda}(v', \beta) dv' d\beta 
+ b^{(l)}(\lambda) ) \\
& = x^{(l)}_0(  \lambda ).
\end{split}
\]
Thus $x_0^{(l)}(u,\alpha,\lambda) = x_0^{(l)}(\lambda)$ for all $l$ (without index $\alpha$ for $l=1$).
The rest of the argument follows from that 
$\| x^{(l)}_c \| = \|x^{(l)} - x^{(l)}_0\| = \|x^{(l)} [ x^{(l-1)}] - x^{(l)} [ x^{(l-1)}_0] \| \le \| x^{(l-1)} -x^{(l-1)}_0  \| = \| x^{(l-1)}_c \|$,
where the inequality is by (a).
\end{proof}

\begin{proof}[Proof of Proposition \ref{prop:deform2}]
We firstly establish that for all $l$,
\begin{equation}\label{eq:proof-deform2-1}
\| x^{(l)} [ T_\rho \circ D_\tau x^{(l-1)}] - T_\rho \circ D_\tau  x^{(l)} [ x^{(l-1)}] \| 
\le 2 c_1 |\nabla \tau|_\infty \|x_c^{(l-1)}\|,
\end{equation}
where $T_\rho$ is replaced by $D_\rho$ if applies to $x^{(0)}$ which does not have index $\alpha$. 
This is because that 
\[
 x^{(l)} [ T_\rho \circ D_\tau x^{(l-1)}]  = T_\rho  x^{(l)} [ D_\tau x^{(l-1)}] 
\]
by Theorem \ref{thm:equiv}, and that
\[
\| T_\rho  x^{(l)} [ D_\tau x^{(l-1)}] - T_\rho \circ D_\tau  x^{(l)} [ x^{(l-1)}] \|
=
\|  x^{(l)} [ D_\tau x^{(l-1)}] -  D_\tau  x^{(l)} [ x^{(l-1)}] \| 
\]
by the definition of $T_\rho$ (a rigid rotation in $u$, and a translation in $\alpha$).
This term can be upper bounded by $c_1 (B_l + C_l) |\nabla \tau|_\infty \|x_c^{(l-1)}\|$ (Lemma \ref{lemma:commuting}),
which leads to the desired bound under (A2) by Proposition \ref{prop:Bl-Cl-Dl}. 

The rest of the proof is similar to that of Proposition 3.3 of \cite{qiu2018dcfnet}: 
Write $x^{(l)}[D_\rho\circ D_\tau x^{(0)}]  - T_\rho \circ D_\tau x^{(l)}[ x^{(0)} ]$ as 
the sum of the differences $ x^{(l)} [ x^{(j)}[D_\rho\circ D_\tau x^{(j-1)}] ] - x^{(l)} [  T_\rho \circ D_\tau x^{(j)}[ x^{(j-1)} ] ]$
for $j=1, \cdots, l$.
The norm of the $j$-th term is bounded by 
$ \| x^{(j)}[D_\rho\circ D_\tau x^{(j-1)}] -   T_\rho \circ D_\tau x^{(j)}[ x^{(j-1)} ]  \| $ due to Proposition \ref{prop:l2stable} (a),
which, by applying \eqref{eq:proof-deform2-1} together with Proposition \ref{prop:l2stable} (b), can be bounded by 
$2 c_1 |\nabla \tau|_\infty \|x^{(0)}\|$. Summing over $j$ gives the claim. 
\end{proof}

\begin{proof}[Proof of Theorem \ref{thm:deform1}]
The proof is similar to that of Theorem 3.8 of \cite{qiu2018dcfnet}. 
With the bound in Proposition \ref{prop:deform2}, it suffices to show that 
\[
\| T_\rho \circ D_\tau  x^{(L)} [ x^{(0)}]  - T_\rho   x^{(L)} [ x^{(0)}]  \|
\le c_2 2^{-j_L} |\tau|_\infty \|x^{(0)}\|. 
\]
By the definition of $T_\rho$, the l.h.s. equals 
$\| D_\tau  x^{(L)} [ x^{(0)}]  -  x^{(L)} [ x^{(0)}]  \|$, 
which can be shown to be less than 
$c_2  |\tau|_\infty D_L \|x_c^{(l-1)}\|$ by extending the proof of Proposition 3.4 of \cite{qiu2018dcfnet}, 
similar to the argument in proving Lemma \ref{lemma:commuting}. 
The desired bound then follows by that $D_L \le 2^{-j_L} A_L \le 2^{-j_L} $ (Proposition \ref{prop:Bl-Cl-Dl} and (A2))
and that $\|x_c^{(l-1)}\| \le  \|x^{(0)}\|$ (Proposition \ref{prop:l2stable} (b)).
\end{proof}

\begin{lemma}\label{lemma:commuting}
In a rotation-equivariant CNN, 
 $B_l$, $C_l$ defined as in \eqref{eq:def-Bl-Cl-Dl}, 
 under (A1), (A3), 
for all $l > 0 $, with $c_1 =4$, $x_c^{(l)}$ as in Proposition \ref{prop:l2stable},
\[
\| x^{(l)} [ D_\tau x^{(l-1)}] - D_\tau  x^{(l)} [ x^{(l-1)}] \| \le c_1 (B_l + C_l) |\nabla \tau|_\infty \|x_c^{(l-1)}\|.
\]
\end{lemma}
\begin{proof}[Proof of Lemma \ref{lemma:commuting}]
The proof is similar to that of Lemma 3.2 of \cite{qiu2018dcfnet}. 
Specifically, when $l=1$, the argument is the same, 
making use of the fact that $\int |w(\Theta_\alpha v )|dv = \int |w( v )|dv, \quad \forall \alpha \in S^1$
and $\int_{S^1} d\alpha =1 $ due to the normalization of $\frac{1}{2\pi}$.
When $l > 1$, the same technique applies 
by considering the joint integration of $\int_{\R^2} \int_{S^1} (\cdots) dv d\beta$ instead of just $dv$.
The only difference is in using the new definitions of $B^{(l)}_{\lambda', \lambda}$ and $C^{(l)}_{\lambda', \lambda}$ for $l>1$
as in \eqref{eq:def-B-C-D-1},
both of which involve the integration of $\int_{S^1} (\cdots) d\beta$. The detail is omitted.
\end{proof}

\section{Experimental Details in Section 4}

\subsection{Object recognition with rotMNIST and CIFAR10}

In the experiments on rotMNIST dataset, the network architecture is shown in Table \ref{tab:network-arch-1}.
Stochastic gradient descent (SGD) with momentum is used to train 100 epochs with decreasing learning rate from $10^{-2}$ to $10^{-4}$.

In the experiments on CIFAR10 dataset, the VGG16-like network architecture is shown in Table \ref{tab:network-arch-2}.
SGD with momentum is used to train 100 epochs with decreasing learning rate from $10^{-2}$ to $10^{-4}$.

\subsection{Convolutional Auto-encoder for image reconstruction}

The network architecture is shown in Table \ref{tab:network-arch-3}. The network is trained on 50,000 training samples, the training set is  augmented by rotating each sample at 8 random angles,
producing 400k training set. The network is trained for 10 epochs, where the learning rate decreases from $10^{-3}$ to $10^{-6}$.

\subsection{Face recognition on Facescrub}\label{sec:A3-face}

To facilitate the evaluation on both known and unknown subjects, we 
select the first 500 of the 530 identities as our training subjects. 
The remaining 30 subjects are used for validating out of sample performance, namely the unknown subjects. 
The experiment on unknown subjects is critical for face models to generate over unseen people.
For both known and unknown subjects, we hold 10 images from each person as the probe images, and the remaining as the gallery images. 
The images are preprocessed by aligning facial landmarks using \cite{landmark} and crop the aligned face images to $112 \times 112$ with color. 
Thus, both our CNN and RotDCF models are trained with near-frontal and upright-only face images.

The network architecture is shown in Table \ref{tab:facenet}. 
According to the formula in Section \ref{sec:2},
the number of trainable parameters in the RotDCF Net is about $(\frac{1}{2})^2\cdot \frac{K}{L^2} \cdot K_\alpha = \frac{1}{4}$ of that of the CNN.


\begin{thebibliography}{10}

\bibitem{chen2015compressing}
Wenlin Chen, James Wilson, Stephen Tyree, Kilian Weinberger, and Yixin Chen.
\newblock Compressing neural networks with the hashing trick.
\newblock In {\em International Conference on Machine Learning}, pages
  2285--2294, 2015.

\bibitem{cheng2016rifd}
Gong Cheng, Peicheng Zhou, and Junwei Han.
\newblock Rifd-cnn: Rotation-invariant and fisher discriminative convolutional
  neural networks for object detection.
\newblock In {\em Computer Vision and Pattern Recognition (CVPR), 2016 IEEE
  Conference on}, pages 2884--2893. IEEE, 2016.

\bibitem{Cohen2016}
Taco~S. Cohen and Max Welling.
\newblock Group equivariant convolutional networks.
\newblock In {\em {ICML}}, pages 2990--2999, 2016.

\bibitem{cohen2018spherical}
Taco~S. Cohen, Mario Geiger, Jonas Koehler, and Max Welling.
\newblock Spherical CNNs.
\newblock {\em arXiv preprint arXiv:1801.10130}, 2018.

\bibitem{cohen2016steerable}
Taco~S. Cohen and Max Welling.
\newblock Steerable CNNs.
\newblock {\em arXiv preprint arXiv:1612.08498}, 2016.

\bibitem{DentonZBLF14}
Emily~L. Denton, Wojciech Zaremba, Joan Bruna, Yann LeCun, and Rob Fergus.
\newblock Exploiting linear structure within convolutional networks for
  efficient evaluation.
\newblock In {\em {NIPS}}, pages 1269--1277, 2014.

\bibitem{freeman1991design}
William~T. Freeman, Edward~H. Adelson, et~al.
\newblock The design and use of steerable filters.
\newblock {\em IEEE Transactions on Pattern analysis and machine intelligence},
  13(9):891--906, 1991.

\bibitem{Marcos2016}
Diego~M. Gonzalez, Michele Volpi, and Devis Tuia.
\newblock Learning rotation invariant convolutional filters for texture
  classification.
\newblock {\em CoRR}, abs/1604.06720, 2016.

\bibitem{hallman2015oriented}
Sam Hallman and Charless~C. Fowlkes.
\newblock Oriented edge forests for boundary detection.
\newblock In {\em Proceedings of the IEEE Conference on Computer Vision and
  Pattern Recognition}, pages 1732--1740, 2015.

\bibitem{han2015deep_compression}
Song Han, Huizi Mao, and William~J. Dally.
\newblock Deep compression: Compressing deep neural networks with pruning,
  trained quantization and huffman coding.
\newblock {\em International Conference on Learning Representations (ICLR)},
  2016.

\bibitem{han2015learning}
Song Han, Jeff Pool, John Tran, and William~J. Dally.
\newblock Learning both weights and connections for efficient neural network.
\newblock In {\em Advances in Neural Information Processing Systems}, pages
  1135--1143, 2015.

\bibitem{hinton2011transforming}
Geoffrey~E. Hinton, Alex Krizhevsky, and Sida~D. Wang.
\newblock Transforming auto-encoders.
\newblock In {\em International Conference on Artificial Neural Networks},
  pages 44--51. Springer, 2011.

\bibitem{howard2017mobilenets}
Andrew~G. Howard, Menglong Zhu, Bo Chen, Dmitry Kalenichenko, Weijun Wang,
  Tobias Weyand, Marco Andreetto, and Hartwig Adam.
\newblock Mobilenets: Efficient convolutional neural networks for mobile vision
  applications.
\newblock {\em arXiv preprint arXiv:1704.04861}, 2017.

\bibitem{SqueezeNet}
Forrest~N. Iandola, Song Han, Matthew~W. Moskewicz, Khalid Ashraf, William~J.
  Dally, and Kurt Keutzer.
\newblock Squeezenet: Alexnet-level accuracy with 50x fewer parameters and
  $<$0.5mb model size.
\newblock {\em arXiv:1602.07360}, 2016.

\bibitem{jaderberg2015spatial}
Max Jaderberg, Karen Simonyan, Andrew Zisserman, et~al.
\newblock Spatial transformer networks.
\newblock In {\em Advances in neural information processing systems}, pages
  2017--2025, 2015.

\bibitem{Jaderberg2015}
Max Jaderberg, Karen Simonyan, Andrew Zisserman, and Koray Kavukcuoglu.
\newblock Spatial transformer networks.
\newblock In {\em {NIPS}}, pages 2017--2025, 2015.

\bibitem{jaderberg2014speeding}
Max Jaderberg, Andrea Vedaldi, and Andrew Zisserman.
\newblock Speeding up convolutional neural networks with low rank expansions.
\newblock {\em arXiv preprint arXiv:1405.3866}, 2014.

\bibitem{landmark}
Vahid Kazemi and Josephine Sullivan.
\newblock One millisecond face alignment with an ensemble of regression trees.
\newblock In {\em Proceedings of the 2014 IEEE Conference on Computer Vision
  and Pattern Recognition, CVPR}, 2014.

\bibitem{Kivinen2011}
Jyri~J. Kivinen and Christopher K.~I. Williams.
\newblock Transformation equivariant boltzmann machines.
\newblock In {\em {ICANN}}, pages 1--9, 2011.

\bibitem{cifar}
Alex Krizhevsky.
\newblock Learning multiple layers of features from tiny images.
\newblock Technical report, 2009.

\bibitem{laptev2016ti}
Dmitry Laptev, Nikolay Savinov, Joachim~M. Buhmann, and Marc Pollefeys.
\newblock Ti-pooling: transformation-invariant pooling for feature learning in
  convolutional neural networks.
\newblock In {\em Proceedings of the IEEE Conference on Computer Vision and
  Pattern Recognition}, pages 289--297, 2016.

\bibitem{lebedev2014speeding}
Vadim Lebedev, Yaroslav Ganin, Maksim Rakhuba, Ivan Oseledets, and Victor
  Lempitsky.
\newblock Speeding-up convolutional neural networks using fine-tuned
  cp-decomposition.
\newblock {\em arXiv preprint arXiv:1412.6553}, 2014.

\bibitem{maninis2016convolutional}
Kevis-K. Maninis, Jordi Pont-Tuset, Pablo Arbel{\'a}ez, and Luc Van~Gool.
\newblock Convolutional oriented boundaries.
\newblock In {\em European Conference on Computer Vision}, pages 580--596.
  Springer, 2016.

\bibitem{ng2014data}
Hong-Wei Ng and Stefan Winkler.
\newblock A data-driven approach to cleaning large face datasets.
\newblock In {\em Image Processing (ICIP), 2014 IEEE International Conference
  on}, pages 343--347. IEEE, 2014.

\bibitem{oyallon2015deep}
Edouard Oyallon and St{\'e}phane Mallat.
\newblock Deep roto-translation scattering for object classification.
\newblock In {\em CVPR}, volume~3, page~6, 2015.

\bibitem{papyan2017convolutional}
Vardan Papyan, Yaniv Romano, and Michael Elad.
\newblock Convolutional neural networks analyzed via convolutional sparse
  coding.
\newblock {\em The Journal of Machine Learning Research}, 18(1):2887--2938,
  2017.

\bibitem{vgg-face}
O.~M. Parkhi, A.~Vedaldi, and A.~Zisserman.
\newblock Deep face recognition.
\newblock In {\em British Machine Vision Conference}, 2015.

\bibitem{qiu2018dcfnet}
Qiang Qiu, Xiuyuan Cheng, Robert Calderbank, and Guillermo Sapiro.
\newblock DCFNet: Deep neural network with decomposed convolutional filters.
\newblock {\em ICML} 2018,
\newblock {\em arXiv:1802.04145}.

\bibitem{rigamonti2013learning}
Roberto Rigamonti, Amos Sironi, Vincent Lepetit, and Pascal Fua.
\newblock Learning separable filters.
\newblock In {\em Computer Vision and Pattern Recognition (CVPR), 2013 IEEE
  Conference on}, pages 2754--2761. IEEE, 2013.

\bibitem{rubinstein2010double}
Ron Rubinstein, Michael Zibulevsky, and Michael Elad.
\newblock Double sparsity: Learning sparse dictionaries for sparse signal
  approximation.
\newblock {\em IEEE Transactions on signal processing}, 58(3):1553--1564, 2010.

\bibitem{schmidt2012learning}
Uwe Schmidt and Stefan Roth.
\newblock Learning rotation-aware features: From invariant priors to
  equivariant descriptors.
\newblock In {\em Computer Vision and Pattern Recognition (CVPR), 2012 IEEE
  Conference on}, pages 2050--2057. IEEE, 2012.

\bibitem{Schmidt2012}
Uwe Schmidt and Stefan Roth.
\newblock Learning rotation-aware features: From invariant priors to
  equivariant descriptors.
\newblock In {\em {CVPR}}, pages 2050--2057, 2012.

\bibitem{sifre2013rotation}
Laurent Sifre and St{\'e}phane Mallat.
\newblock Rotation, scaling and deformation invariant scattering for texture
  discrimination.
\newblock In {\em Computer Vision and Pattern Recognition (CVPR), 2013 IEEE
  Conference on}, pages 1233--1240. IEEE, 2013.

\bibitem{simonyan2014very}
Karen Simonyan and Andrew Zisserman.
\newblock Very deep convolutional networks for large-scale image recognition.
\newblock {\em arXiv preprint arXiv:1409.1556}, 2014.

\bibitem{tai2015convolutional}
Cheng Tai, Tong Xiao, Yi Zhang, Xiaogang Wang, et~al.
\newblock Convolutional neural networks with low-rank regularization.
\newblock {\em arXiv preprint arXiv:1511.06067}, 2015.

\bibitem{weiler2017learning}
Maurice Weiler, Fred~A. Hamprecht, and Martin Storath.
\newblock Learning steerable filters for rotation equivariant cnns.
\newblock {\em arXiv preprint arXiv:1711.07289}, 2017.

\bibitem{worrall2017harmonic}
Daniel~E. Worrall, Stephan~J. Garbin, Daniyar Turmukhambetov, and Gabriel~J.
  Brostow.
\newblock Harmonic networks: Deep translation and rotation equivariance.
\newblock In {\em Proc. IEEE Conf. on Computer Vision and Pattern Recognition
  (CVPR)}, volume~2, 2017.

\bibitem{Wu2015}
Fa Wu, Peijun Hu, and Dexing Kong.
\newblock Flip-rotate-pooling convolution and split dropout on convolution
  neural networks for image classification.
\newblock {\em CoRR}, abs/1507.08754, 2015.

\bibitem{CAM}
Bolei Zhou, Aditya Khosla, {\`{A}}gata Lapedriza, Aude Oliva, and Antonio
  Torralba.
\newblock Learning deep features for discriminative localization.
\newblock In {\em Proceedings of the 2014 IEEE Conference on Computer Vision
  and Pattern Recognition, CVPR}, 2016.

\bibitem{zhou2017oriented}
Yanzhao Zhou, Qixiang Ye, Qiang Qiu, and Jianbin Jiao.
\newblock Oriented response networks.
\newblock In {\em 2017 IEEE Conference on Computer Vision and Pattern
  Recognition (CVPR)}, pages 4961--4970. IEEE, 2017.

\end{thebibliography}
\end{document}